\documentclass{cmslatex}
\usepackage[paperwidth=7in, paperheight=10in, margin=.875in]{geometry}
 \usepackage[backref,colorlinks,linkcolor=red,anchorcolor=green,citecolor=blue]{hyperref}
\usepackage{amsfonts,amssymb}
\usepackage{amsmath}
\usepackage{graphicx}
\usepackage{cite}
\usepackage{enumerate}
\renewcommand{\theequation}{\arabic{section}.\arabic{equation}}

\usepackage{graphicx}
\usepackage{caption}
\usepackage{amsfonts}
\usepackage{amsmath}
\usepackage{amssymb}
\usepackage{fancyhdr}
\usepackage{titlesec}
\usepackage{indentfirst}
\usepackage{booktabs}
\usepackage{verbatim}
\usepackage{color}
\usepackage{tcolorbox}
\usepackage{mathtools}
\usepackage{bm}
\usepackage{wasysym}
\usepackage{empheq}
\usepackage{dsfont}
\usepackage[font=small,labelfont=bf]{caption} 
\usepackage{chngcntr}
\newcommand{\rom}[1]{\uppercase\expandafter{\romannumeral #1\relax}}
\usepackage[page,header]{appendix}
\usepackage{subfig}
\usepackage{titletoc}

\newenvironment{bluetext}{\color{black}}{\ignorespacesafterend}

\usepackage{algorithm}
\usepackage{algorithmicx}
\usepackage{algpseudocode}
\numberwithin{equation}{section}
\newcommand{\argmin}{\operatornamewithlimits{argmin}}
\newcommand{\argmax}{\operatornamewithlimits{argmax}}

\newcommand*\Let[2]{\State #1 $\gets$ #2}

\def\thetheorem {{\arabic{section}.\arabic{theorem}}}

\usepackage{tikz}

\def\tr{\mathrm{tr}}

\def\k{\kappa}
\def\b{\beta}

\def\g{\gamma}
\def\lam{\lambda}

\def\p{\phi}

\def\v{\varepsilon}

\def\p{\rho}

\def\star{\circ}

\def\A{\mathbb A}

\def\E{\mathbb E}

\def\R{\mathbb R}
\def\S{\mathbb S}

\def\mM{\mathcal M}

\def\l{\left}
\def\r{\right}

\def\ll{\left\lVert}
\def\rl{\right\rVert}
\def\lv{\left\lvert}
\def\rv{\right\rvert}

\def\ds{\displaystyle}

\def\h{\hat}
\def\t{\tilde}

\def\tcb{\textcolor{black}}

\newcommand{\hZ}{\h{Z}}
\newcommand{\hY}{\h{Y}}
\newcommand{\hz}{\h{z}}

\newcommand{\hp}{\h{p}}
\newcommand{\hC}{\h{C}}
\newcommand{\hD}{\h{D}}
\newcommand{\hB}{\h{B}}
\newcommand{\hG}{\h{G}}
\newcommand{\hH}{\h{H}}
\newcommand{\hA}{\h{A}}

   \allowdisplaybreaks
\begin{document}
 \title{
 Operator Shifting for Model-based Policy Evaluation
\thanks{Received date, and accepted date (The correct dates will be entered by the editor).}}


        \author{Xun Tang\thanks{Institute for Computational and Mathematical Engineering, Stanford
        University, Stanford, CA 94305, (xuntang@stanford.edu)} 
        \and Lexing Ying\thanks{Department of
        Mathematics and Institute for Computational and Mathematical Engineering, Stanford University,
        Stanford, CA 94305, (lexing@stanford.edu).} 
        \and Yuhua Zhu\thanks{Department of Mathematics,
        Stanford University, Stanford, CA 94305, (yuhuazhu@stanford.edu).}}

         \pagestyle{myheadings} \markboth{Operator Shifting for Policy Evaluation}{X. Tang, L. Ying, AND Y. Zhu} \maketitle

          \begin{abstract}
        In model-based reinforcement learning, the transition matrix and reward vector are often estimated
        from random samples subject to noise. Even if the estimated model is an unbiased estimate of the
        true underlying model, the value function computed from the estimated model is biased. We introduce
        an operator shifting method for reducing the error introduced by the estimated model. When the
        error is in the residual norm, we prove that the shifting factor is always positive and upper
        bounded by $1+O\l({1}/{n}\r)$, where $n$ is the number of samples used in learning each row of the
        transition matrix. We also propose a practical numerical algorithm for implementing the operator
        shifting.
        \end{abstract}
        


\section{Introduction}
\label{sec:Intro}

Reinforcement learning (RL) has received much attention following recent successes, such as AlphaGo
and AlphaZero \cite{silver2016mastering, silver2017mastering}. One of the fundamental problems of RL
is policy evaluation \cite{sutton2018reinforcement}. When the transition dynamics are unknown, one
learns the dynamics model from observed data in model-based RL.  However, even if the learned model
is an unbiased estimate of the true dynamics, the policy evaluation under the learned model is
biased. The question of interest in this paper is whether one can increase the accuracy of the
policy evaluation given an estimated dynamics model.


We consider a discounted Markov decision process (MDP) $\mM = (\S,\A, P, r, \g)$ with discrete state
space $\S$ and discrete action space $\A$. $\lv \S \rv$ and $|\A|$ are used to denote the size of
$\S$ and $\A$, respectively. $P$ is a third-order tensor, where for each action $a\in\A$,
$P^a\in\R^{\lv \S \rv\times\lv \S \rv}$ is the transition matrix between the states. $r$ is a
second-order tensor that $r_{s,a}$ is the reward at state $s\in\S$ if action $a\in\A$ is
taken. Finally, $\g\in(0,1)$ is the discount factor.  A policy $\pi$ is a second-order tensor,
where for each state $s\in\S$, $\pi_s$ represents the probability distribution over $\A$.  At each
time step $t$, one observes a state $s_t\in\S$ and takes an action $a_t\in\A$ according to the
policy $\pi_{s_t}$. The environment returns the next state $s_{t+1}$ according to the distribution
$P^{a_t}_{s_t,\cdot}$ and an associated reward $r_{s_t,a_t}$. The state value function
$v^\pi\in\R^{\lv \S \rv}$ is the expected discounted cumulative reward if one starts from an initial
state $s$ and follows a policy $\pi$, i.e., the $s$-th component is 
\[v^\pi_s = \underset{\substack{a_t\sim\pi_{s_t}\\
  s_{t+1}\sim P^{a_t}_{s_t,\cdot}}}{\E}\l[\sum_{t\geq0} \g^tr_{s_t,a_t} | s_0 = s\r].\]

Given a policy $\pi$, the goal of policy evaluation in MDP is to solve for $v^\pi$. Let
$b^\pi\in\R^{\lv \S \rv}, P^\pi\in\R^{\lv \S \rv\times\lv \S \rv}$ be the reward vector and
the transition matrix under policy $\pi$, i.e.,
\begin{equation}\label{def bpi}
  b^\pi = \sum_{a} r_{sa}\pi^a_s, \quad P^\pi = \sum_a P^a_{ss'}\pi^a_s.    
\end{equation}
The value function $v^\pi$ satisfies the Bellman equation \cite{sutton2018reinforcement} $\left(I-\g
P^{\pi}\right)v^{\pi} = b^{\pi}$.  For notational simplicity, we drop the dependency on $\pi$ and
write this system as
\begin{equation}\label{eq: origin}
    (I - \g P)v = b.
\end{equation}

In practice, the true transition matrix $P$ and the reward vector $b$ are often inaccessible. In the
model-based RL, one approximates the transition matrix $P$ and the reward vector $b$ by the
empirical data $\hat{P}$ and $\hat{b}$ estimated from samples, respectively
\cite{levine2013guided,deisenroth2011pilco,sutton1991dyna,watter2015embed, oh2015action}. A naive
approach is to solve
\begin{equation}\label{eq: def of V}
  \l(I-\g\hat{P} \r) \hat{v} = \hat{b}.
\end{equation}
Even if $\hat{P}$ and $\hat{b}$ are unbiased estimates for $P$ and $b$, $\hat{v} =
\l(I-\g\hat{P} \r)^{-1} \hat{b}$ is a biased estimate for $v$, i.e.,
$\E_{\hat{P},\hat{b}}\hat{v}\neq v$.


The operator shifting idea was introduced in \cite{etter2020operator,etter2021operator} to
address this issue. The paper \cite{etter2020operator} considers the noisy symmetric elliptic
systems, while the follow-up paper \cite{etter2021operator} addresses the asymmetric setting under
the assumption that $\h{b}$ is isotropic, i.e., $\E[\h{b} \h{b}^\top]=I$. However, this
isotropic condition often fails to hold in RL. In this paper, we extend the operator shifting
framework to general MDPs of form \eqref{eq: origin}. When applying this framework to the MDP
setting, we add an appropriately chosen matrix $\hat{K}$ to the operator $\l(I-\g \hat{P} \r)^{-1}$
so that the shifted estimate $\t{v}= \l[\l(I-\g \hat{P} \r)^{-1} - \b\hat{K}\r]\h{b}$ is a better
estimate than $\hat{v}$ in the sense that,
\begin{equation}\label{eq: goal in intro}
    \E_{\h{P},\h{b}} \ll\t{v} - v\rl^2 < \E_{\h{P},\h{b}} \ll \h{v} - v\rl^2
\end{equation} 
for a certain norm $\ll \cdot \rl$. 





\paragraph{Contributions.}
We derive a stable shifted operator for model-based policy evaluation without assumptions on
the underlying transition dynamics or reward vectors.  When the approximated transition matrix
$\h{P}$ follows the multinomial distribution and $n$ samples are used to learn each row of the
transition matrix $P$,
\begin{itemize}
\item we prove that the optimal shifting factor is always positive and upper bounded by
  $1+O\l(\frac{1}{n}\r)$ for any $P$ and $b$, which guarantees the stability of the shifted
  operator, and
\item we propose a numerical algorithm to find the optimal shifting factor, which is more
  efficient and accurate than the bootstrapping method proposed in \cite{etter2020operator}.
\end{itemize}



\paragraph{Related work.}
Our problem is a special instance of the larger field of uncertainty quantification (UQ). In most UQ
problems, one assumes that the operator (linear or non-linear) and the source term are generated
from known distributions, and the task is to estimate certain quantities (such as moments, tail
bounds) of the distribution of the solution. A large variety of numerical methods have been
developed in UQ for this purpose in the last two decades
\cite{ghanem2003stochastic,gunzburger2014stochastic,
  le2010spectral,xiu2010numerical,peherstorfer2018survey,estep2006fast,stuart2010inverse,marzouk2016introduction},
including Monte-Carlo and quasi Monte-Carlo methods
\cite{niederreiter2012monte,mishra2012sparse,barth2011multi,dick2013high, graham2011quasi},
stochastic collocation methods
\cite{babuvska2007stochastic,nobile2008sparse,xiu2005high,back2011stochastic}, stochastic Galerkin
methods \cite{babuska2004galerkin,xiu2002wiener,cohen2010convergence,najm2009uncertainty}, and etc.
The problem that we face is somewhat different: since the true $P$ and $b$ are unknown, one does not
know the distributions of the empirical data $\hat{P}$ or $\hat{b}$. As a result, the solution
relies more on statistical techniques such as shrinkage \cite{james1992estimation} rather than the
traditional UQ techniques.

\paragraph{Contents.}
Section \ref{sec: model} derives the oracle estimators and proves their lower and upper bounds.
Section \ref{sec: numerical examples} proposes a practical estimator and demonstrates its
performance with a few numerical examples. Section \ref{sec:proofs} contains all the proofs of the
main results.


\section{Operator Shifting for Policy Evaluation}
\label{sec: model}


\subsection{Problem setup.}\label{sec: formalism}
As mentioned above, $P$ and $b$ are the {\em unknown} underlying transition matrix and reward
vector, while $\h{P}$ and $\h{b}$ are unbiased estimates for $P$ and $b$, respectively. For
notational convenience, we introduce $A$ and $\h{A}$
\begin{equation}\label{def of A hatA}
  A = \l(I - \g P\r), \quad \h{A} = \l(I - \g \hat{P}\r).
\end{equation}
Since the transition dynamics is not symmetric in general, both $P$ and $A$ are non-symmetric. The norm of interest is a slightly generalized version of the residual norm
\begin{equation}\label{eq: def of norm}
  \ll x \rl^2_{M} = x^\top A^\top M Ax,
\end{equation}
where $M$ is a symmetric positive definite matrix. This paper mainly discusses two cases: (1) $M=I$, which means $\ll \cdot \rl_{M}$ is the usual residual norm, and (2) $M = A^{-\top}A^{-1}$, which means
$\ll\cdot\rl_{M}$ is the $l_{2}$ norm.

In this paper, we choose the shifting matrix $\h{K} =\h{A}^{-1}$, which implies that the
shifted estimate is $\l( 1-\b \r) \h{A}^{-1} \h{b}$. By using $\v = \l( 1-\b \r)$ instead as the
shifting parameter, one can write the above estimate as $\t{v}_\v \equiv \v \h{A}^{-1} \h{b}$, and the objective is to minimize the following mean square error over $\v$,
\begin{equation}\label{eqn: objective function for operator shifting}
  \mathrm{MSE}(\varepsilon) \equiv  \E_{\hat{P},\hat{b}}\ll v - \t{v}_\v \rl_{M}^{2}.
\end{equation}
The minimizer $\v^*$ to \eqref{eqn: objective function for operator shifting} is referred as the {\em optimal shifting factor}.

Since \eqref{eqn: objective function for operator shifting} is a quadratic minimization, one can
explicitly write out the optimal shifting factor $\v^*$:
\begin{equation}\label{eqn: eqn1 for operator shifting}
    \v^* = \frac{
    \mathbb{E}_{\hat{P},\hat{b}}\left[
    b^{\top}MA\hat{v}
    \right]
    }{
    \mathbb{E}_{\hat{P},\hat{b}}\left[
    \hat{v}^\top A^{\top}MA\hat{v}
    \right]
    },
    \quad\text{where}\; \h{v} = \hat{A}^{-1}\h{b}.
\end{equation}



The oracle estimate \eqref{eqn: eqn1 for operator shifting} is not easy to work with as it
depends on the unknown matrix $A$. Our immediate goal is to derive a closed-form approximation of
$\v^*$, which is accurate and allows for efficient implementation. To achieve this, we introduce a
second-order approximation $\v^\star$ to $\v^*$. We show that $\v^\star$ takes a simple closed-form
without approximating any expectations under the following mild assumption:
\begin{itemize}
    \item [] {\bf Assumption 1}: The $i$-th row \tcb{$\h{p}^{\top}_i$} of $\h{P}$ is an \tcb{random
      vector} $\frac{1}{n_{i}}X_i$, where $n_{i}$ is the number of samples for state \(i\) and $X_i$
      follows the multinomial distribution with \tcb{$\E[X_i] = n_{i}p^{\top}_i$}. \tcb{Moreover, \(X_{i}\) is independent from \(X_{j}\) whenever \(i \not = j\).} The \(i\)-th entry of \(\h{b}\) is an average of observed reward at state \(i\).
      
\end{itemize}
The part of Assumption 1 on the estimation of \(P\) is equivalent to that \tcb{$\h{p}^{\top}_i$} follows the
normalized multinomial distribution, which holds when a tabular maximum likelihood model
\cite{sutton1991dyna} is used to estimate the transition dynamics $P$. That is, one generates
sufficiently many transitions according to $P$ and lets $\h{P}_{ii'} = n_{ii'}/n_{i},$ where
$n_{ii'}$ denotes the number of transitions observed from $i$ to $i'$, and $n_i=\sum_{i'}n_{ii'}$.

Throughout this paper, we assume for simplicity that the number of samples $n_{i}$ of each state is
the same, i.e., for any $i \in \S$, $n_{i} \equiv n$.  The sample size $n$ plays an important role
in determining the magnitude of the operator shifting factor $\v^{*}$ and the performance of the
operator shifting algorithm. If the value of $n_i$ depends on $i\in\S$, all the theoretical
results still hold with slight modification (see Remark \ref{rmk: diff_ns} for details).


\subsection{Second-order approximation for $\v^*$.}\label{sec: 2nd order}
To simplify the discussion, we introduce $\hZ$ and $\hY$
\begin{equation}\label{def of Z Y}
  \hZ = A - \h{A} = \g \left(P - \hat{P} \right),\quad
  \hY = \hZ A^{-1} = \l( A - \h{A}\r)A^{-1},
\end{equation}
where $A$ and $\hat{A}$ are defined in \eqref{def of A hatA}. Some basic algebraic manipulations
lead to the following lemma.


\begin{lemma}\label{lemma: vstar in Y}
    When  $\ds\E\left[\hat{b}\right] = b$, the optimal shifting factor $\v^{*}$ defined in \eqref{eqn: eqn1 for
      operator shifting} has the form
    \begin{equation}\label{eq: vstar in Y}
      \v^{*}
        =
        \frac{
            \mathbb{E}_{\hat{P}}\left[
            b^{\top}M\l(I- \hY\r)^{-1}b\right]
        }{
            \mathbb{E}_{\hat{P}}\left[\tr\left(\l(\mathrm{cov}\left[\hat{b}\right] + b^{\top}b\r)\l(I- \hY\r)^{-\top}M\l(I- \hY\r)^{-1}\right)\right]
        },
    \end{equation}
    where $\hY$ is defined in \eqref{def of Z Y}.  Moreover, if the values of the reward at state $s$
    and $s'$ (i.e. $\h{b}_s$ and $\h{b}_{s'}$) are uncorrelated, the matrix $\mathrm{cov}[\h{b}]$ is
    diagonal.
\end{lemma}

Next, we approximate the value of $\v^{*}$ using a Neumann expansion of the matrix $\l(I-\hY \r)^{-1}$
\begin{equation}\label{eq: neumann series}
  \l(I- \hY\r)^{-1} = I + \hY + \hY^{2} + O\l(\frac{\rho(\hY)^{3}}{1 - \rho(\hY)}\r),
\end{equation}
when the spectral radius $\rho(\hY)<1$. In fact, a modest requirement on $n$ guarantees
$\rho(\hY)<1$ with high probability, as shown in Appendix \ref{appendix: condition for convergence
  of Neumann Series}.  The denominator term in \eqref{eq: vstar in Y} admits the approximation
\begin{equation}\label{eq: inv(I-Y)'*inv(I-Y)}
    \l(I- \hY\r)^{-\top}M\l(I- \hY\r)^{-1} \approx M + M\l(\hY+\hY^2\r) + (\hY^{\top} + (\hY^{\top})^2)M +  \hY^{\top}M\hY.
\end{equation}
Assumption 1 implies $\mathbb{E}\left[\hY\right] = \mathbb{E}\left[\hY^{\top}\right] = 0$ as a simple
consequence of $\h{P}$ being an unbiased estimator. Therefore, after taking an expectation, the first
order terms of $\hY$ in \eqref{eq: neumann series} and \eqref{eq: inv(I-Y)'*inv(I-Y)} disappear.

We can further approximate the shifting factor $\v^*$ by expanding $\l(I- \hY\r)^{-1}$ in the
numerator and denominator of \eqref{eq: vstar in Y} up to the second order in $\hY$. When Assumption
1 holds and $\rho(\hY)<1$, the approximated optimal shifting factor $\v^{*}$ defined in
\eqref{eqn: eqn1 for operator shifting} has a second-order approximation
\begin{equation}\label{eq: vstar}
  \v^{*} \approx \v^{\star} \equiv \frac{
    \mathbb{E}_{\hat{P}}\left[
      b^{\top} ( M + \frac{M\hY^2+ (\hY^{\top})^{2}M}{2}) b
      \right]
  }{
    \mathbb{E}_{\hat{P}}\left[
      \tr\left(\left(\mathrm{cov}\left[\hat{b}\right] + b^{\top}b\right)
      ( M + \hY^{\top}M\hY + M\hY^{2} + (\hY^{\top})^2M)\right)\right]}.
\end{equation}
The derivation of \eqref{eq: vstar} is deferred to Section \ref{sec: proof of closed-form theorem}.

Under Assumption 1, this second-order approximation can be written in a form without explicit
expectation. This expectation-free form depends on the transition matrix $P$, the expected reward
$b$, and the reward covariance $\mathrm{cov}\left[\hat{b}\right]$. Let $\hp_{i}$ be random vectors
corresponding to the i-th row of $\h{P}$ and $p_i=\mathbb{E}\left[\hp_{i}\right]$, i.e.,
$\{\hp_i\}_{i=1}^{\lv \S \rv}$ and $\{p_i\}_{i=1}^{\lv \S \rv}$ are the row vectors of $\h{P}$ and
$P$, respectively:
\begin{equation}\label{eq: def of Xi pi}
  \hat{P} = \begin{bmatrix}
    \hp_{1}^{\top}\\
    \dots\\
    \hp_{\lv \S \rv}^{\top}
  \end{bmatrix},\quad
  P = \begin{bmatrix}
    p_{1}^{\top}\\
    \dots\\
    p_{\lv \S \rv}^{\top}
  \end{bmatrix}.
\end{equation}

\begin{theorem}\label{thm: structure of Y'Y and Y^2}
  The second-order approximation $\v^\star$ in \eqref{eq: vstar} admits the expectation-free form
  \begin{equation}\label{eq: vstar in G H}
    \v^{\star} = \theta(b,P) \equiv \frac{
      b^{\top} ( M + H/2) b
    }{
      b^{\top} ( M + G + H) b
      + \tr\left(\mathrm{cov}\left[\hat{b}\right]
      ( M + G + H)\right)},
  \end{equation}
  where
  \begin{equation}\label{eq: def of G and H}
    \begin{cases}
      B_{i} = \frac{1}{n}\l(\diag{p_{i}} - p_{i}p_{i}^{\top}\r);\\
      G = \mathbb{E}_{\hat{P}}\left[\hY^{\top}M\hY\right]=
      \g^2 A^{-\top}\left(\sum_{i=1}^{\lv \S \rv}\left[M\right]_{ii}B_{i}\right)A^{-1};\\
      \begin{aligned}
        H &= \mathbb{E}_{\hat{P}}\left[
        (\hY^{\top})^{2}M
        +
        M\hY^2
        \right] \\
        &= \g^2\left[\sum_{i = 1}^{\lv \S \rv}A^{-\top}B_{i}A^{-1}\mathrm{diag}\l(e_{i} \r)\right]M + \g^2M\left[\sum_{i = 1}^{\lv \S \rv}\mathrm{diag}\l(e_{i} \r)A^{-\top}B_{i}A^{-1}\right]
      \end{aligned}
      .
    \end{cases}
  \end{equation}
  Here $\mathrm{diag}(e_i)\in\R^{\lv \S \rv\times\lv \S \rv}$ is a matrix with elements $0$ except
  for $1$ on the $(i,i)$-th entry, $p_i$ are row vectors of $P$ as defined in \eqref{eq: def of
    Xi pi}, and the matrices $A$ and $\hY$ depend on $P$.
\end{theorem}
The proof of the above theorem is given in Section \ref{sec: proof of closed-form theorem}. 

\begin{remark}\label{rmk: diff_ns}
  Theorem \ref{thm: structure of Y'Y and Y^2} still holds under conditions weaker than Assumption
  1. Assuming the rows of $\h{P}$ and entries for \(\h{b}\) are independent unbiased estimators, then
  the second-order approximation in \eqref{eq: vstar} is still valid. Moreover, the expectation-free
  form in \eqref{eq: vstar in G H} holds when one replaces the definition of $B_{i}$ in \eqref{eq:
    def of G and H} by $B_{i} = \mathrm{cov}\left[\hp_{i}\right]$. This slightly more general
  statement is presented in Lemma \ref{lemma: auxiliary lemma for YMY and Y^2}, from which Theorem
  \ref{thm: structure of Y'Y and Y^2} is derived as a special case. In particular, if the state $i$
  receives $n_{i}$ samples, then \eqref{eq: vstar in G H} will still hold with $B_{i}$ in \eqref{eq:
    def of G and H} replaced by $\frac{1}{n_{i}}\l(\diag{p_{i}} - p_{i}p_{i}^{\top}\r)$.
\end{remark}

Theorem \ref{thm: structure of Y'Y and Y^2} also proves that the choice of $\v^{\star}$ is
asymptotically as powerful as $\v^*$ with $n\to\infty$. For $\mathrm{MSE}(\varepsilon)$ defined in \eqref{eqn:
  objective function for operator shifting}, the following estimation holds, with the proof
deferred to Section \ref{sec: proof of corollary on error asymptotics}.
\begin{lemma}\label{lemma: structure of objective function}
The MSE in \eqref{eqn: objective function for operator shifting} can be approximated by
  \begin{equation}\label{eq: objective under Taylor second-order}
    \mathrm{MSE}(\varepsilon) =  \left(1 - \varepsilon\right)^2 \ll b \rl_{M}^{2} + ( g + h + t)\varepsilon^2 - h\varepsilon + O\l(n^{-\frac{3}{2}}\r),
  \end{equation}
  where $g = b^{\top}Gb, h = b^{\top}Hb, t = \tr\left(\mathrm{cov}\left[\hat{b}\right] \l( M+G+H\r)\right)$
  with all symbols defined in Theorem \ref{thm: structure of Y'Y and Y^2}.
  In addition,
  \begin{equation}\label{eq: v subtracts vstar}
    \v^* - \v^{\star} = O\l(n^{-\frac{3}{2}}\r). 
  \end{equation}
\end{lemma}
\tcb{
The relative error reduction factor
$\eta\equiv\frac{\mathrm{MSE}(1)-\mathrm{MSE}(\v^{\star})}{\mathrm{MSE}(1)}$, with MSE representing
the mean square error defined in \eqref{eqn: objective function for operator shifting}, is a
useful measure for improvements.  Below is a corollary of Lemma \ref{lemma: structure of objective
  function} regarding $\eta$.
\begin{corollary}\label{corollary: structure of error reduction}
Define relative error reduction factor as
$\eta\equiv\frac{\mathrm{MSE}(1)-\mathrm{MSE}(\v^{\star})}{\mathrm{MSE}(1)}$. For \(n\) sufficiently large, \(\eta\) is positive, and decays as follows
  \begin{equation}
    \eta = O\l(1/n \r),
  \end{equation}
  where MSE is defined in \eqref{eqn: objective function for operator shifting}.
\end{corollary}
The proof is given in Section \ref{sec: proof of corollary on error asymptotics}. The numerical
results also verify the relationship in the above Corollary.
}

\begin{bluetext}

\subsection{Lower and upper bounds for $\v^\star$.}\label{sec: lower bound}
In this section, we aim to provide bounds to show that $\varepsilon^{\star}$ will
approximately fall in the $(0,1)$ range.  Throughout this subsection, we conduct the analysis in the
residual norm, which is the case with $M = I$. We present an upper bound and a lower bound for
$\v^\star$ in Theorem \ref{thm: upper bd}. The relevant parameters are $\lv\S \rv$, $n$ and $\g$, which are the number of state, the number of samples per state and
the discount factor, respectively.



\begin{theorem}\label{thm: upper bd}
  Let $p_M = \max P_{i,j}$ and $b_M = \max_i |b_i|$. If
  $\frac{p_M}{n}\frac{\g^2}{\l(1-\g\r)^2}\l(\frac{\l(1-\g\r)}{\g} + \frac{\sqrt{\lv \S \rv} b_M}{\ll b \rl_2}\r)^2 \leq \frac{1}{2}$,
  then $\v^\star$ is bounded by
  \[
  0 < \v^\star \leq 1 + \frac{p_M}{n} \frac{\g^2}{\l(1-\g\r)^2} \l(\l(1-\g\r) +\g\frac{\sqrt{\lv \S \rv} b_M}{\ll b \rl_2}\r)^2.
  \]

\end{theorem}
The bound comes from technique using the spectral
structure of the covariance matrix of a multinomial distribution and a tight bound for $(I-\g P)^{-1}b$. We defer the proof to Section \ref{sec:proofs}.

We now discuss the implication of Theorem \ref{thm: upper bd}. The reward vector $b$ is {\em spread} if
$\frac{\max_i |b_i|}{\ll b \rl_2} \sim O\l(\lv\S \rv^{-1/2}\r)$. Similarly, the transition matrix
$P$ is {\em spread} if $\max P_{ij}\sim o\l(\lv \S \rv^{-1/2}\r)$. If both $b$ and \(P\) are spread, it follows that \(\epsilon^{\star} - 1\) is upper bounded by a term in \(O(n^{-1} \lv \S \rv^{-1/2})\).

\end{bluetext}

\section{Practical algorithm}\label{sec: numerical examples}


\subsection{Algorithm.}\label{sec: algo}

In practice, we do not have direct access to $P$ or $b$. Therefore, the second-order estimate
$\v^\star$ derived in \eqref{eq: vstar in G H} is an oracle estimator. One can address this issue by
bootstrapping the distribution of $\h{P}$. More specifically, let \(P\) be a transition matrix and
denote $\mathcal{M}_{n}(P)$ as the normalized multinomial distribution that the estimated transition
matrix \(\hat{P}\) follows according to Assumption 1. Since one only has access to a single
observation $\hat{P}$, $\mathcal{M}_{n}(P)$ is approximated by $\mathcal{M}_{n}(\h{P})$ in the
numerical implementation.  In the usual bootstrapping procedure, one needs to simulate
i.i.d. samples $\{\Tilde{P}_{(j)} \}_{j=1}^{l} \sim \mathcal{M}_{n}(\h{P})$. By setting
$\Tilde{Y}_{(j)}=\gamma(\Tilde{P}_{(j)}-\hat{P})\hat{A}^{-1}$ and following the form in Theorem \ref{thm: structure of Y'Y and Y^2}, one can
approximate $\v^\star$ in \eqref{eq: vstar} by replacing the expectation with an empirical mean:
\begin{equation}\label{eq: bootstrap vstar}
  \v^{\star} \approx  \frac{
    \h{b}^{\top} ( M + \Tilde{H}/2) \h{b}
  }{
    \h{b}^{\top} \l( M + \Tilde{G} + \Tilde{H}\r) \h{b}
    + \tr\left(\t{\Sigma}
    \l( M + \Tilde{G} + \Tilde{H}\r)\right)},
\end{equation}
with
\begin{equation}
  \begin{cases}
    \Tilde{G} = \frac{1}{l}\sum_{j=1}^{l}
    \Tilde{Y}_{(j)}^{\top} M \Tilde{Y}_{(j)} \approx \mathbb{E}_{\h{P}}\left[\hY^{\top}M\hY\right] = G ;\\
    \Tilde{H} = \frac{1}{l}\sum_{j=1}^{l}
    M\Tilde{Y}_{(j)}^2 + \left(\Tilde{Y}_{(j)}^{\top}\right)^2 M
    \approx \mathbb{E}_{\h{P}}\left[M \hY^2+ (\hY^{\top})^{2}M\right] = H;\\
    \t{\Sigma} \approx \mathrm{cov}\left[\hat{b}\right].
  \end{cases}
\end{equation}
However, there is a major drawback to this scheme. In addition to the error caused by the difference between $\mathcal{M}_{n}(\h{P})$ and $\mathcal{M}_{n}({P})$, the scheme introduces additional errors  due to the empirical mean in
place of the expectation. The empirical mean errors $\Tilde{G} - G$ and $\Tilde{H} - H$
are of order $O\l(l^{-1/2}\r)$. In addition, the procedure in \eqref{eq: bootstrap vstar} has a computational cost of order
$O\l(l\lv \S \rv^3\r)$. 


\paragraph{Plug-in estimate.}
Luckily in our case, Assumption 1 (i.e., $\h{P}$ follows the normalized multinomial distribution)
allows for a direct formula for $\v^{\star}$, which automatically removes the error in the empirical
mean.  We can simply set
\begin{equation}\label{eq: definition of bootsrapped vstar}
  \Tilde{\v}^{\star} := \theta(\h{b},\h{P}),
\end{equation}
where $\theta(b,P)$ is defined in Theorem \ref{thm: structure of Y'Y and Y^2}. The complete
numerical algorithm is presented in Algorithm \ref{alg:operator shifting with multinomial
  estimation MDP}. The right-hand side of \eqref{eq: bootstrap vstar} converges to
$\theta(\h{b},\h{P})$ as $l \rightarrow \infty$. In addition, the computational cost is reduced from
$O\l(l\lv \S \rv^3 \r)$ to $O\l(\lv \S \rv^3 \r)$. This complete removal of empirical mean error is
what sets the multinomial MDP case apart from general operator shifting.  Moreover, since both
$\v^{\star}$ in \eqref{eq: vstar in G H} and $\Tilde{\v}^{\star}$ in \eqref{eq: definition of bootsrapped vstar} share the same functional form, the lower and upper bounds in Section \ref{sec: lower bound} automatically
apply to both $\v^{\star}$ and $\Tilde{\v}^{\star}$. In all the following numerical examples, we use the approximated factor \(\t\v^{\star}\), which does not rely on oracle access.


\begin{algorithm}[ht!]
  \caption{Operator Shifting for estimating MDP (Multinomial)
    \label{alg:operator shifting with multinomial estimation MDP}}
    Outputs the bootstrapped $\Tilde{\v}^{\star} = \theta(\h{b},\h{P})$. The function can also output the true value by $\v^{\star} = \theta(b,P)$ if one has oracle access to $P,b$.
  \begin{algorithmic}[1]
    \Require{$\h{P}$: Estimated transition matrix}
    \Require{$\h{b}$: Estimated expected reward}
    \Require{$n$: Sample data size per state}
    \Require{$\Tilde{\Sigma}$: Estimated covariance matrix of $\hat{b}$}
    \Require{$\g$: Discount factor}
    \Require{$M$: The chosen norm matrix}

    \Statex
    \Function{$\theta$}{$\h{P}, \h{b}$}
      \Let{$\hC,\hD$}{$\mathrm{zeros}(\lv \S \rv)$} \Comment{$\mathrm{zeros}(n)$: A zero matrix of size $n$}
      \For{$i \gets 1 \textrm{ to } \lv \S \rv$}
        \Let{$\hp_{i}$}{$ \hat{P}^{\top}e_{i} $} \Comment{Get the $i$-th row of $\h{P}$}
        \Let{$\hB_{i}$}{$\frac{1}{n}\left(\hp_{i}\hp_{i}^{\top} - \mathrm{diag}(\hp_{i}) \right)$} \Comment{Estimate covariance}
        \Let{$\hC$}{$\hC + \hB_{i}\hA^{-1}\mathrm{diag}\l(e_{i} \r)$}
        \Let{$\hD$}{$\hD + [M]_{ii}\hB_{i}$}
      \EndFor
      \Let{$\hC$}{$\g^2\hA^{-\top}\hC$} \Comment{Approximate $\mathbb{E}\left[(\hY^{\top})^{2}\right]$}
      \Let{$\hG$}{$\g^2\hA^{-\top}\hD\hA^{-1}$} \Comment{Approximate $\mathbb{E}\left[\hY^{\top}M\hY\right]$}
      \Let{$\hH$}{$\hC M + M \hC^{\top}$} \Comment{Approximate $\mathbb{E}\left[(\hY^{\top})^{2}M + M\hY^2\right]$}
      \Let{$\Tilde{\v}^{\star}$}{$
        \frac{
          \h{b}^{\top} \l( M + \hH/2\r) \h{b}
        }{
          \h{b}^{\top} \l( M + \hG + \hH\r) \h{b}
          + \tr\left(\Tilde{\Sigma}
          \l( M + \hG + \hH\r)\right)}
        $
      }
      \State \Return{$\Tilde{\v}^{\star}$}
    \EndFunction
  \end{algorithmic}
\end{algorithm}



\subsection{Numerical examples.}

\paragraph{Policy evaluation of an MDP over a circle.}

We first consider an MDP over a discrete state space $\mathbb{S} = \{k\}_{k=0}^{N-1}$
with $N = 64$ and $\g = 0.9$. The transition dynamics are given as below,
\[
\begin{aligned}
  &s_{t+1} \gets s_t + \left(1 + Z_{t} \right)a_{t} \mod{N},\\ 
  &r_{s_{t},a_{t}} = {\sin\l(\frac{2\pi s_{t}}{N}\r)} + a_{t}{\cos\l(\frac{2\pi s_{t}}{N}\r)}/10+ X_{t},
\end{aligned}
\]
where $a_{t} \in \{\pm 1\}$ is drawn from a policy $\pi(a_{t}|s_{t}) = \frac{1}{2} + \frac{1}{5}
a_{t}\sin\l(\frac{2\pi s_{t}}{N}\r)$, and $X_{t} \sim N(0,\delta)$ with
$\delta\in\{0,0.1,0.2\}$. When $\delta = 0$, the reward is deterministic. Here $Z_{t}$ is a random
integer taking values in the set $\{-\sigma,\ldots,\sigma\}$ with equal probability, where
$\sigma\in\{1,2,4\}$. A larger $\sigma$ means each state could transit to more neighboring states
under one step. Figure \ref{fig:1D circle MDP} illustrates the distribution of the estimated \(v(x_{s})\) for \(x_{s} := 0\), see caption for implementation detail. One can see that the shifted value is more concentrated around the ground truth value than the naive solution.

\begin{figure}[ht!]
  \centering
  \includegraphics[width =0.6\linewidth]{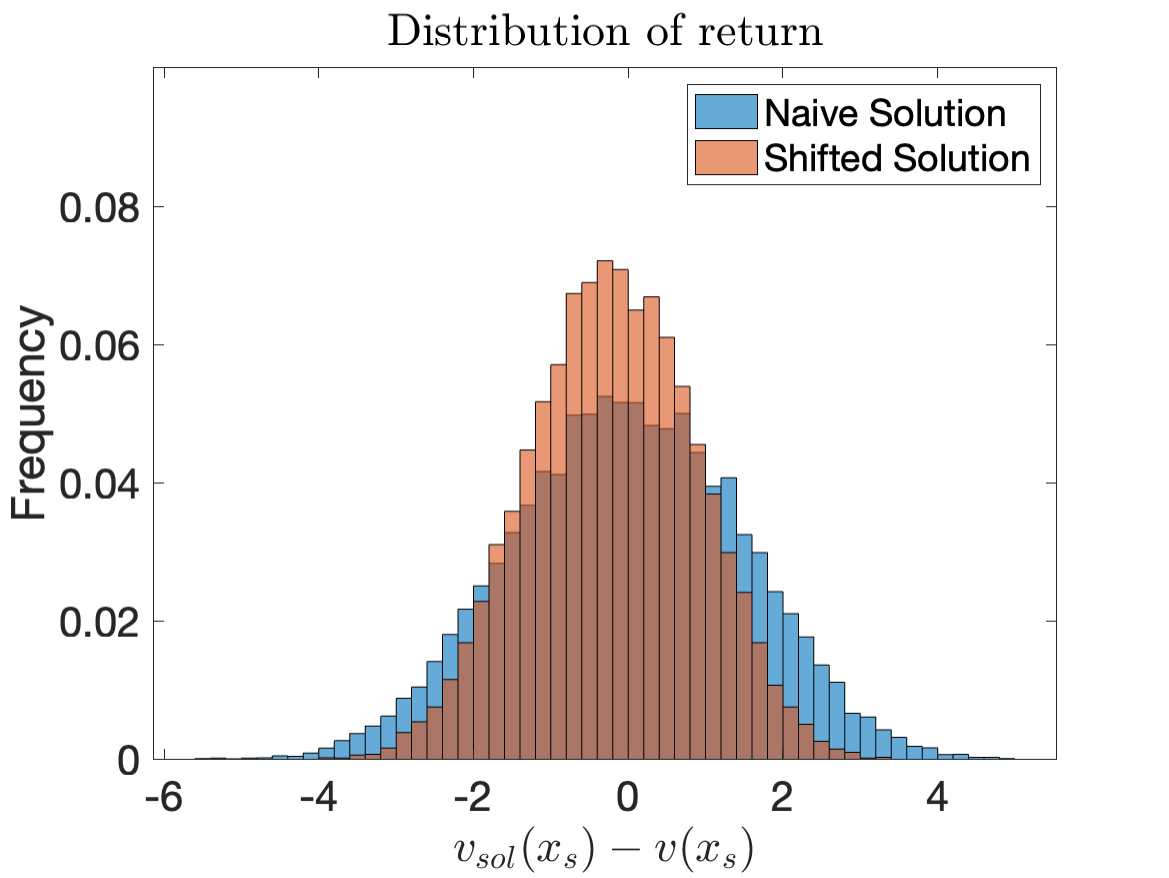}
  \caption{1D circle example. Histogram of \(v_{sol}(x_{s}) - v(x_{s})\), where \(v(x_{s})\) denotes the true return at \(x_{s}\), and \(v_{sol}(x_{s})\) denotes the estimated return from either the naive solution or the operator shifting solution. The shifted solution is more concentrated around the ground truth. The histogram is obtained via \(20,000\) simulations, and we set $n = 8$, $\delta = 0.2$, and $\sigma =
    4$.}
  \label{fig:1D circle MDP}
\end{figure}

\begin{figure}[ht!]
  \centering \subfloat[Residual
    norm]{\includegraphics[width=0.5\linewidth]{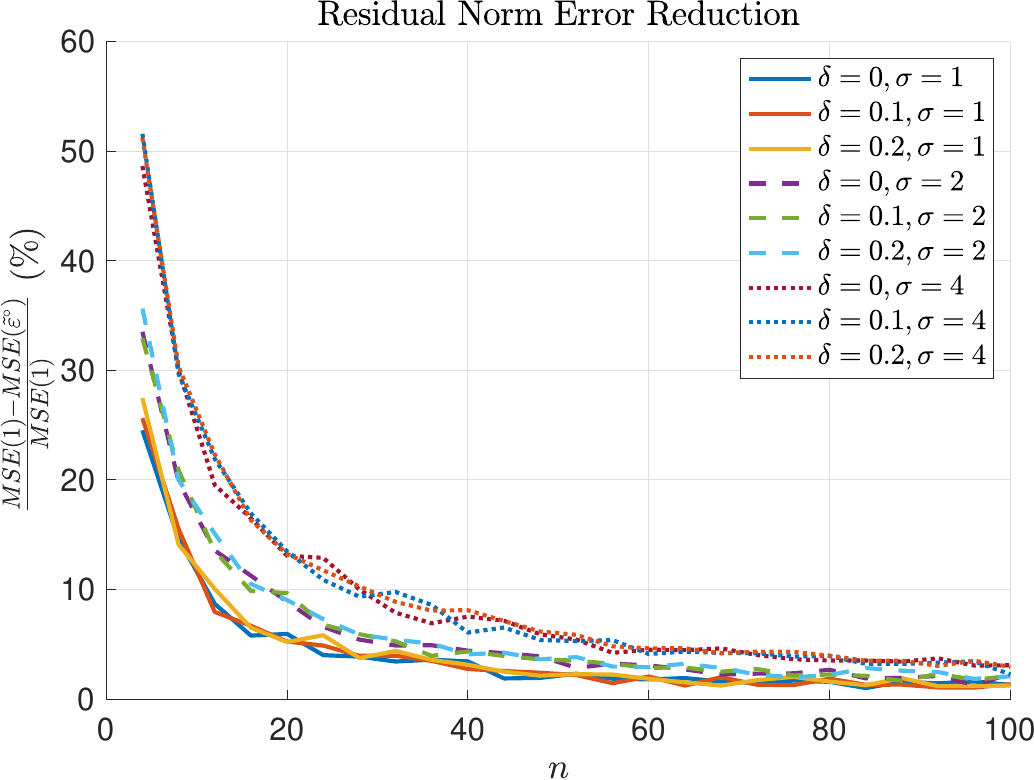}}
  \subfloat[$l_{2}$
    norm]{\includegraphics[width=0.5\linewidth]{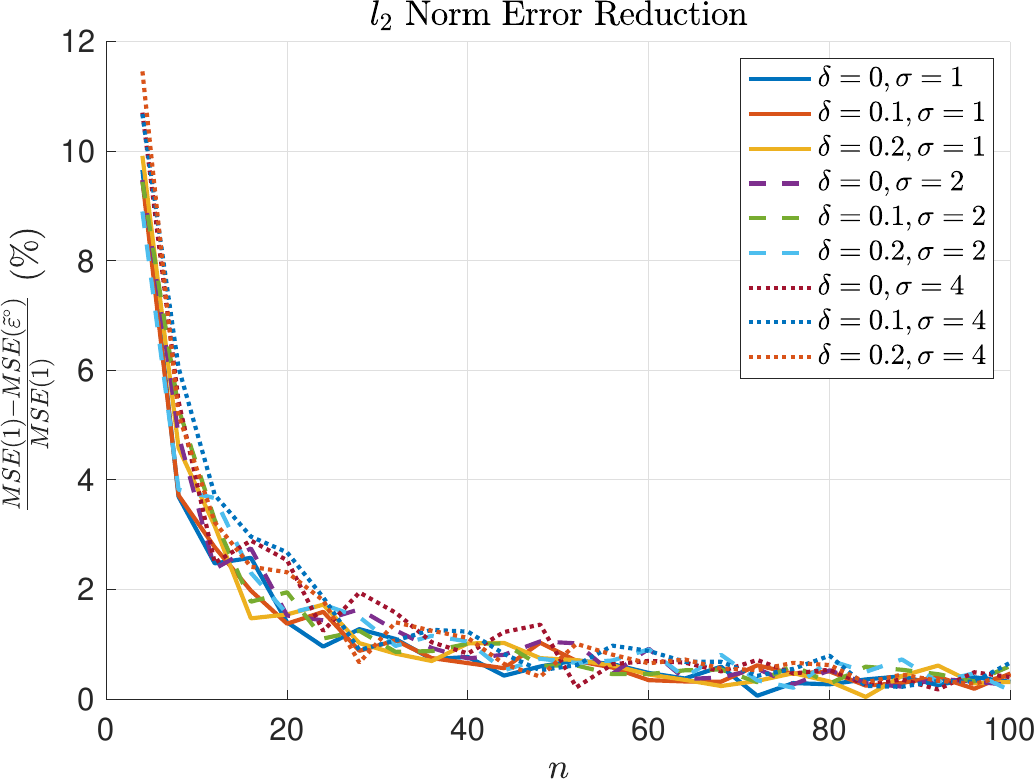}}
  \caption{1D circle example. Error reduction as a function of $n$ for the residual norm (left) and
    the $l_{2}$ norm (right). The $y$-axis is the error reduction rate in MSE, relative to the error
    of the naive solution. The error reduction rate is inversely proportional to $n$ across
    different choices of $\sigma$ and $\delta$. For the residual norm case, the error reduction is
    heavily influenced by the choice of parameters, where a larger $\sigma$ or $\delta$ implies a
    larger reduction in error. }
  \label{fig:1D circle MDP Error trend}
\end{figure}

One can also bootstrap to reduce the $l_{2}$ error. Specifically, the $l_{2}$ error refers to the
case when \(M = A^{-\top}A^{-1}\) in the norm $\ll \cdot \rl_{M}$ defined in \eqref{eq: def of
  norm}. Despite a lack of access to \(A\), one can use \(M = \h{A}^{-\top}\h{A}^{-1}\) for
\(l_{2}\) error minimization, which works well empirically. The error reduction trend remains the
same (see Figure \ref{fig:1D circle MDP Error trend} for details).  Overall, the error reduction for
$l_{2}$ norm is less significant than the residual norm, though it is still significant for small
$n$.

In the following examples, we focus on
the results for the residual norm.

\paragraph{MDPs generated by random graphs.}
To test the robustness of Algorithm \ref{alg:operator shifting with multinomial estimation MDP},
here we apply the operator shifting method to different underlying transition matrices. For
consistency, we set $\lv \S \rv = 64$.

As discussed in the 1D circle case, the randomness in $\hat{b}$ usually boosts the performance of
the operator shifting method. Here we take out the randomness in the reward and instead let
$\hat{b}$ be deterministic, i.e., $\hat{b} = b$ and $\mathrm{cov}[\hat{b}] = 0$. To test different
$b$, we assume that $b$ is randomly generated according to $\mathcal{N}(0,I)$. The transition matrix
$P$ corresponds to the random walk on a directed random graph $G =(V,E,w)$, where \(V = \S\) is the
vertex set, \(E\) is the edge set, and the edge weight is $w:E\rightarrow \mathbb{R}_{\geq 0}$.

Two types of random graphs are considered. In the first {\em dense} case, the graph $G$ is
considered to be fully connected, and the weight $w(e)$ on each edge $e$ is an i.i.d. random variable
following $w(e) \sim \mathcal{U}(0,1)$. In the second {\em sparse} case, a sparse
graph is considered. In order to generate a random sparse graph, one initializes with a graph containing an empty edge set,
\[G \gets G_{0} := (V = \S,E = \varnothing).\] 
For each vertex $v \in V$, two vertices $v_{1},v_{2}$ are randomly selected from the set
\(\S\setminus \{v\}\) that excludes $v$ itself with equal probability, and then \[E \gets E
\cup\{(v_{1},v), (v,v_{2})\}.\] After enumerating over all vertices, one then assigns a weight of
one to all existing edges in \(G\). This construction ensures that none of the vertices is a well or
sink node, {that is, each vertex has at least one indegree and one outdegree}, but the transition
matrix is still quite sparse.

\begin{figure}[ht!]
    \centering
    \subfloat[Random dense MDP]{\includegraphics[width=0.5\linewidth]{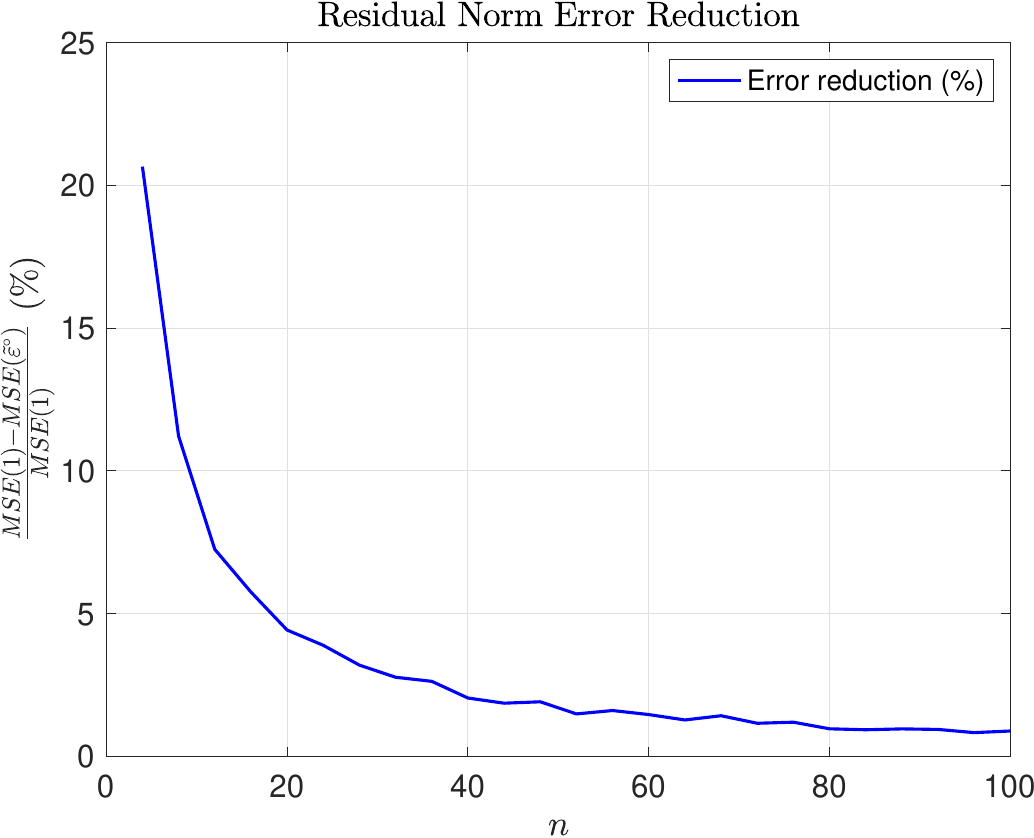}}
    \subfloat[Random sparse MDP]{\includegraphics[width=0.5\linewidth]{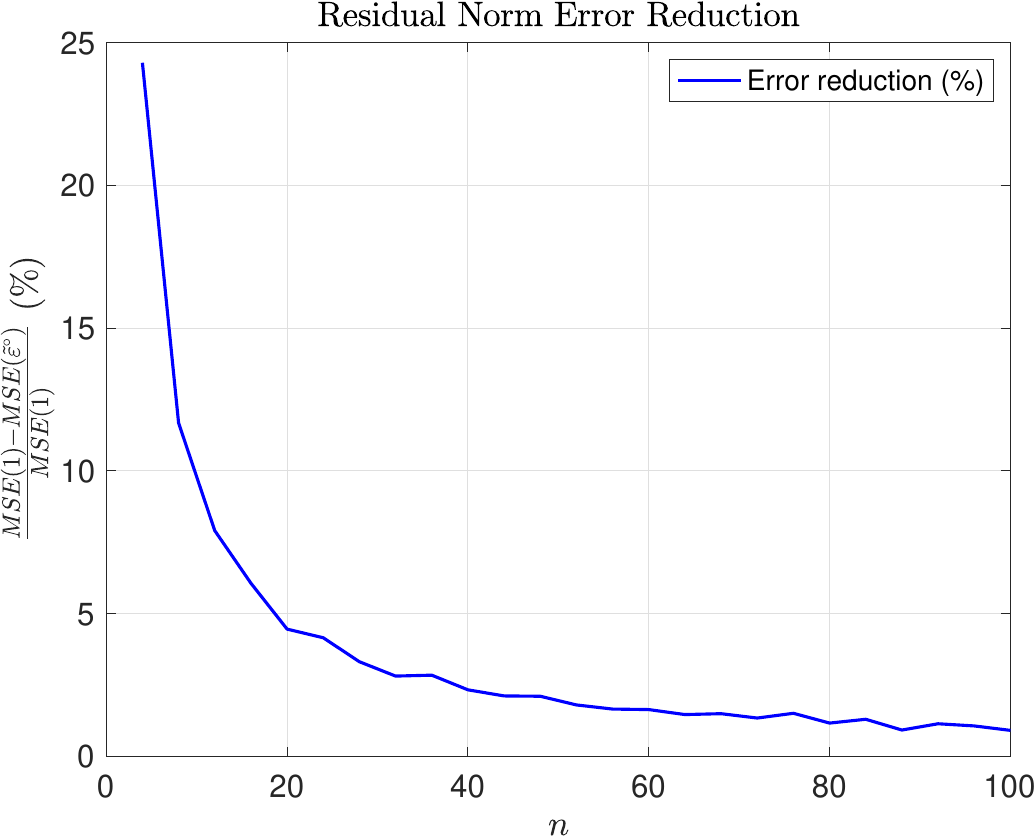}}
    \caption{Random directed graphs. Error reduction as a function of sample size $n$. Left: the random dense
      graph. Right: the random sparse graph.}
    \label{fig:Dense Random MDP Error trend}
\end{figure}
Figure \ref{fig:Dense Random MDP Error trend} shows that the same MSE reduction pattern holds in the
random directed graph cases. The operator shifting solution still consistently outperforms the
naive solution.

\paragraph{Policy evaluation of an MDP over a torus.}

We now consider an MDP with a discrete state space $\mathbb{S} = \{s_{ij} = (i,j)\}_{i,j=0}^{N-1}$
with $N = 8$ and $\g = 0.9$. Note that the size of the state space $|\S|$ is still $64$. Let
$(s)_{k}$ stand for the first or second entry of the vector $s$ with $k = 1$ or $2$. The transition
dynamics and reward are given by 
\[
\begin{aligned}
    &s_{t+1} \gets s_t + (1+ Z_t )a_t \mod N, \\
    &r_{s_{t},a_{t}} \gets 2 + \sin\l(\frac{2\pi (s_{t})_{1}}{N} \r) + \cos\l(\frac{2\pi (s_{t})_{2}}{N}\r) + X_{t},
\end{aligned}
\]
where $ a_t\in\A = \{(\pm 1,0), (0,\pm1)\}$,  $X_{t} \sim N(0,\delta)$ with $\delta \in \{0,0.1,0.2\}$. {Here $Z_{t}$ is a
random integer taking values in the set $\{-\sigma,\ldots,\sigma\}$ with equal probability, where $\sigma \in \{1,2,4\}$. } We use the policy 
\begin{equation}
  \pi(a_{t} = (a_{1},a_{2})|s_{t}) = \frac{1}{4}+ \frac{1}{20} \left( a_{1}\cos\l(\frac{2\pi (s_{t})_{1}}{N} \r) +
  a_{2}\sin\l(\frac{2\pi (s_{t})_{2}}{N} \r)\right).
\end{equation}

Figure \ref{fig:2D Torus Error trend} summarizes the performance and exhibits a similar error
reduction trend. Contrary to the role of the parameters in the 1D circle case, different choices of
$\sigma$ and $\delta$ do not change the performance of the operator shifting method.


\begin{figure}[ht!]
    \centering {\includegraphics[width =0.5\linewidth]{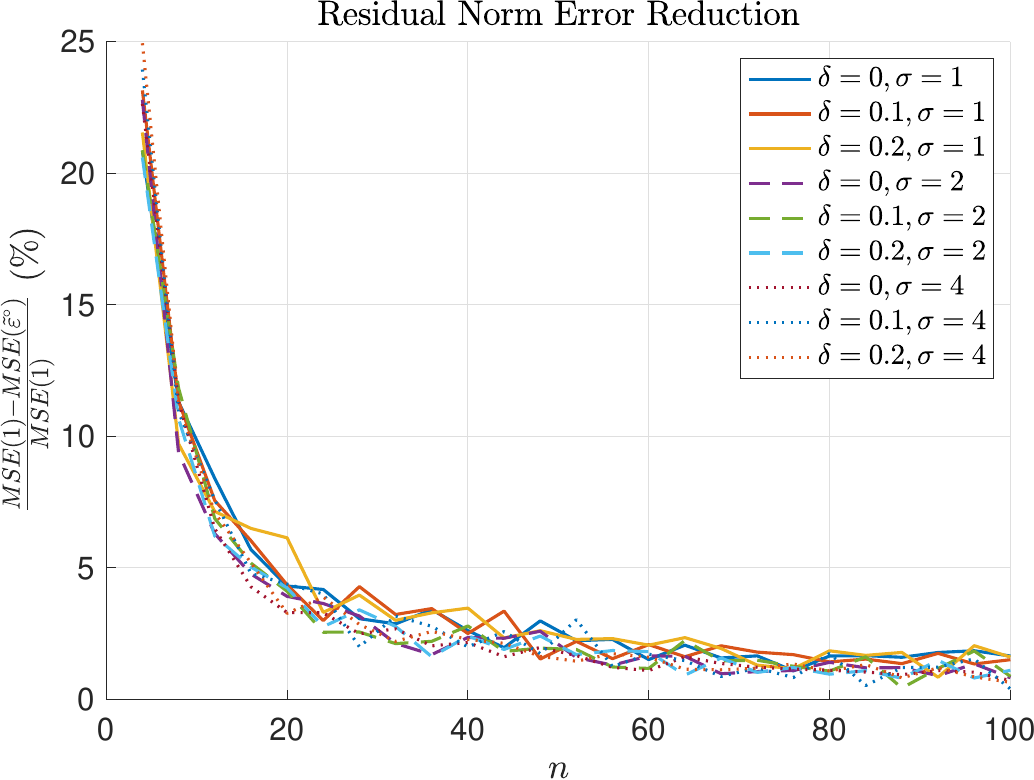}}
    \caption{2D torus example. Error reduction as a function of sample size $n$.}
    \label{fig:2D Torus Error trend}
\end{figure}
\label{sec: numerics}



\paragraph{Summary of numerical experiments.} 
Figure \ref{fig:summary versus trend} plots the normalized MSE of the naive solution against the
operator shifting solution. In the torus and circle cases, the data points are obtained by
varying the sample size $n$, the reward variance $\delta$, and the transition parameter $\sigma$. In the randomly generated MDP case, the data points are
obtained by sampling random MDPs and varying the value of the sample size $n$. The vast majority of
the data points are below the diagonal line, suggesting that operator shifting consistently
reduces the MSE.

\tcb{As a further remark, the numerical result shows that the bounds in Theorem \ref{thm: upper bd} for $\varepsilon^{\star}$ is quite pessimistic. In practice, $\varepsilon^{\star}$ almost always falls in the $(0,1)$ range, even for small $n$.}

\begin{figure}[ht!]
  \centering \subfloat[1D circle]{\includegraphics[width
      =0.4\linewidth]{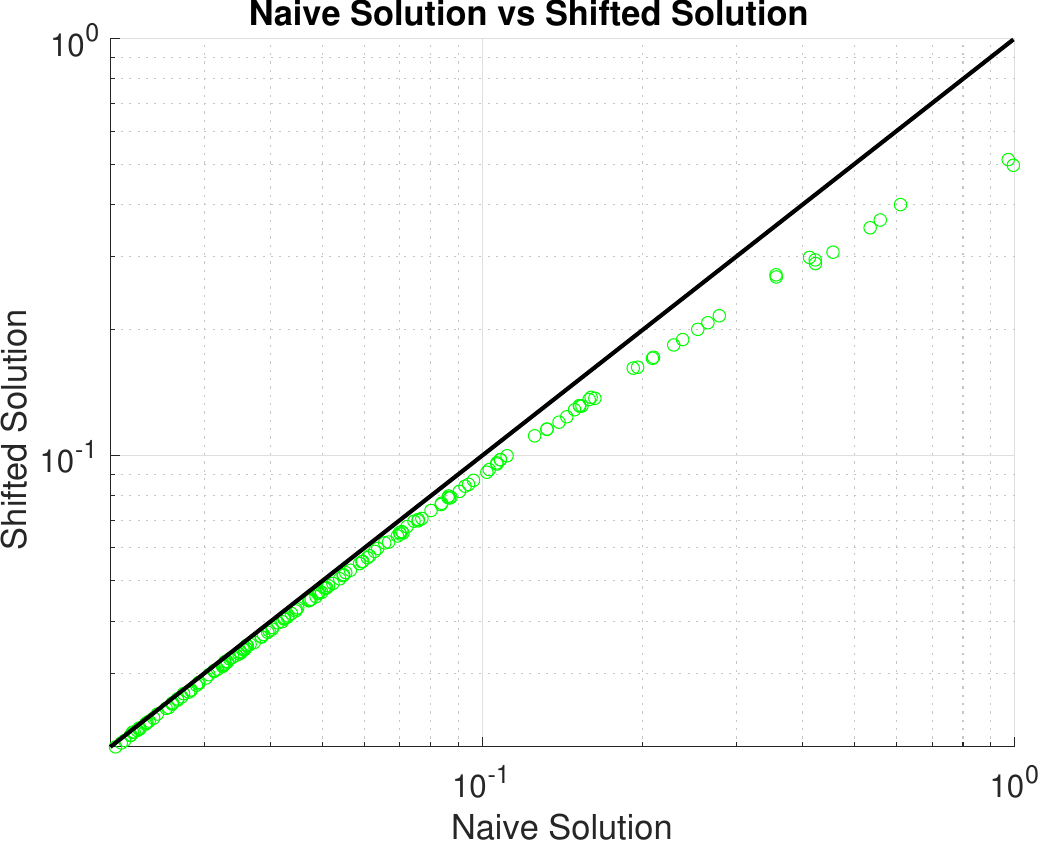}} \subfloat[2D Torus]{\includegraphics[width =0.4\linewidth]{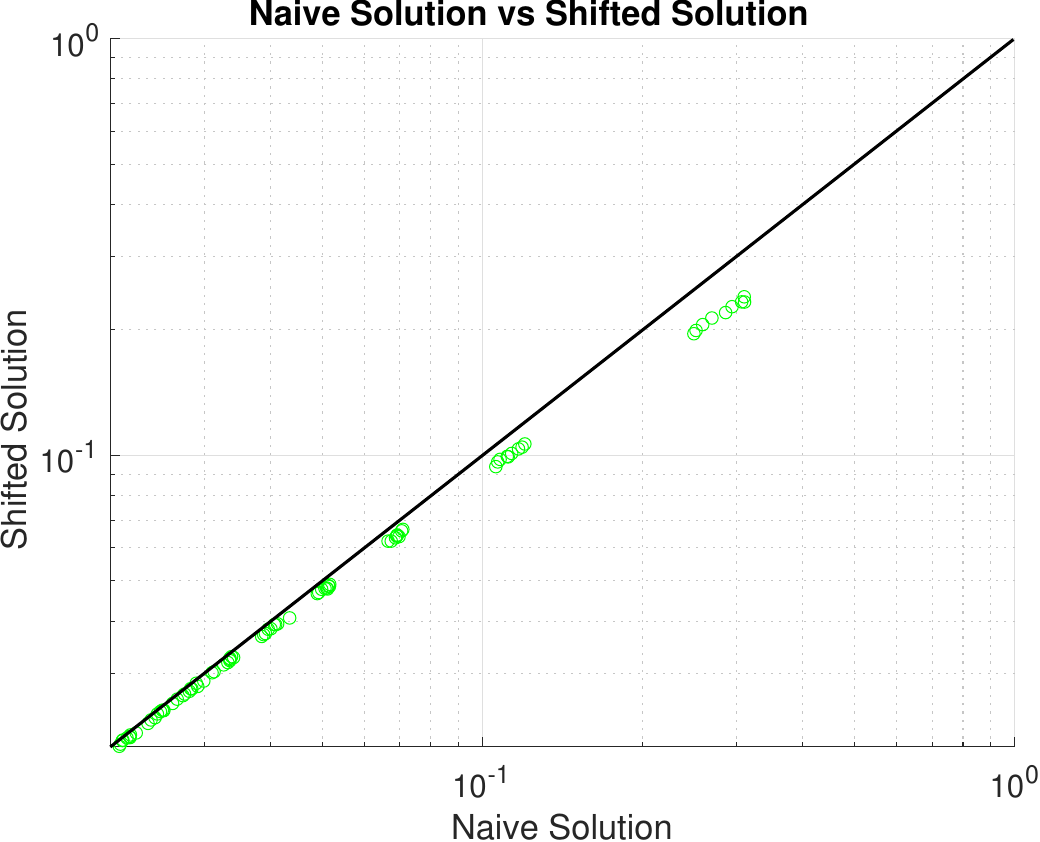}}
  \\ \subfloat[Random dense MDP]{\includegraphics[width
      =0.4\linewidth]{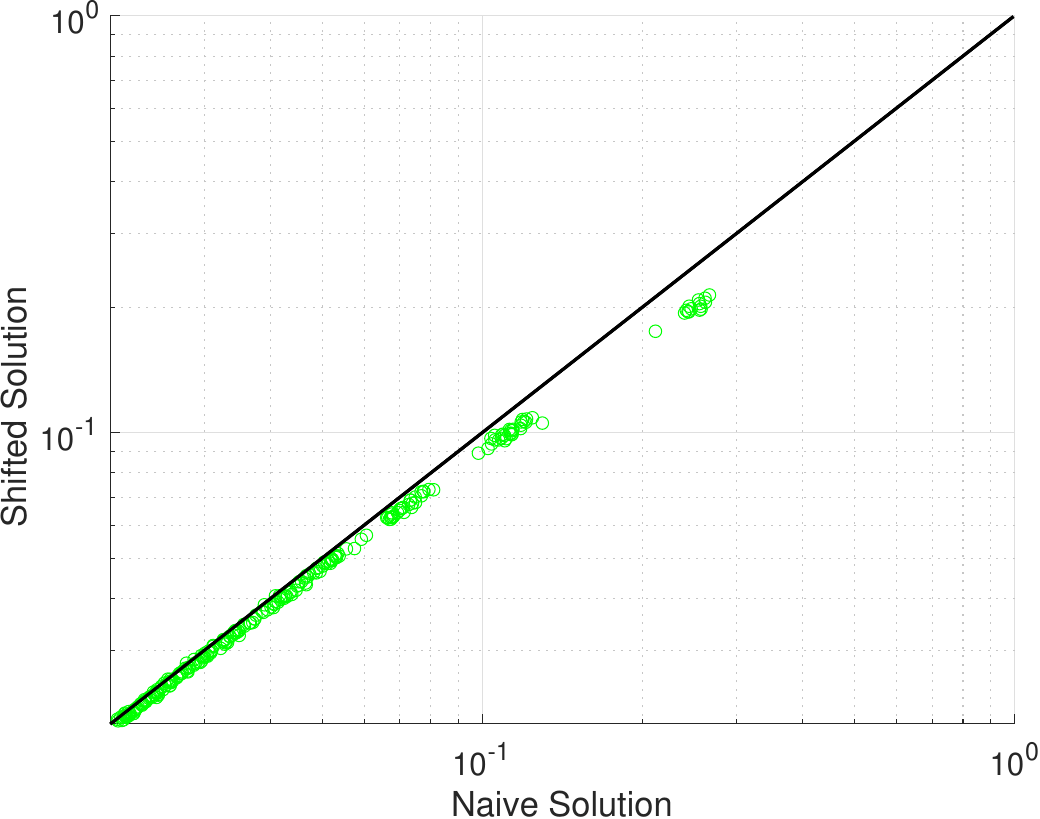}} \subfloat[Random sparse MDP]{\includegraphics[width =0.4\linewidth]{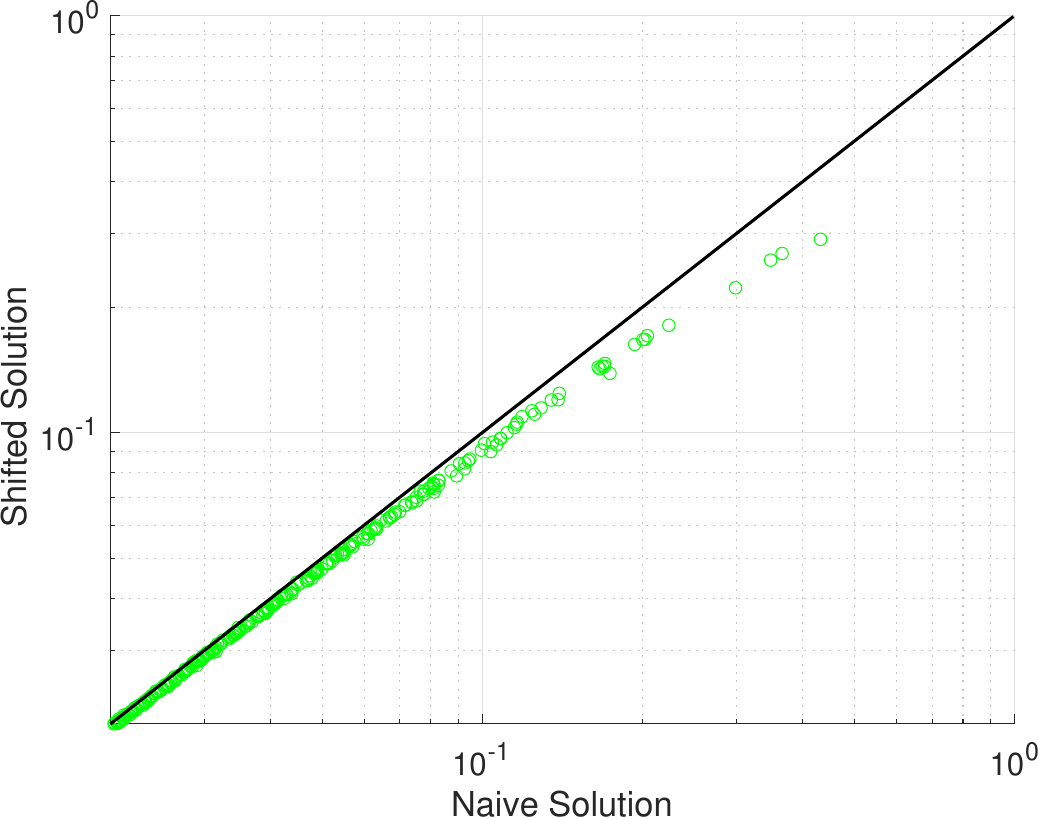}}
  \caption{The normalized MSE of the operator shifting solution is plotted against that of the
    naive solution. All data points are below or close to the diagonal, showing that the operator
    shifting solution outperforms the naive solution in all data points collected.}
  \label{fig:summary versus trend}
\end{figure}

\section{Proofs}\label{sec:proofs}

\subsection{Proof of Lemma \ref{lemma: vstar in Y}.}

\begin{proof}
From \eqref{def of Z Y}, $\l(I - \h{Y}\r)A\hat{v}=\l(\hat{A}A^{-1}\r)A\hat{v}=\hat{A}\hat{v}=\hat{b}$ and
\begin{equation*}
  A\hat{v} = \left(I - \h{Y}\right)^{-1}\hat{b}.
\end{equation*}
Hence (\ref{eqn: eqn1 for operator shifting}) can be written as
\begin{align*}
    \v^{*} &= \frac{
    \mathbb{E}_{\hat{P},\hat{b}}\left[
    b^{\top}M\l(I- \hY\r)^{-1}\hat{b}
    \right]
    }{
    \mathbb{E}_{\hat{P},\hat{b}}\left[
    \hat{b}^{\top}\l(I- \hY\r)^{-\top}M\l(I- \hY\r)^{-1}\hat{b}
    \right]
    }.
\end{align*}
From Assumption 1, $\E[\h{b}] = b$.
Moreover, it follows from Assumption 1 that $\h{P}$ is independent to $\h{b}$. Hence one can write the numerator as
\begin{equation*}
  \mathbb{E}_{\hat{P},\hat{b}}\left[
    b^{\top}M\l(I- \hY\r)^{-1}\hat{b}\right]
    =
    \mathbb{E}_{\hat{P}}\left[
    b^{\top}M\l(I- \hY\r)^{-1}b\right],
\end{equation*}
and the denominator as
\begin{align*}
  &\mathbb{E}_{\hat{P},\hat{b}}\left[
    \hat{b}^{\top}\l(I- \hY\r)^{-\top}M\l(I- \hY\r)^{-1}\hat{b}
    \right]
    \\
    =     &\mathbb{E}_{\hat{P}}\left[    b^{\top}\l(I- \hY\r)^{-\top}M\l(I- \hY\r)^{-1}b    \right]
    +    \mathbb{E}_{\hat{P},\hat{b}}\left[    (b-\hat{b})^{\top}\l(I- \hY\r)^{-\top}M\l(I- \hY\r)^{-1}(b-\hat{b})
    \right].
\end{align*}

One can rewrite the second term in terms of the
variance of $\hat{b}$ by the trace property
\begin{align*}
    &\mathbb{E}_{\hat{P},\hat{b}}\left[
  (b-\hat{b})^{\top}\l(I- \hY\r)^{-\top}M\l(I- \hY\r)^{-1}(b-\hat{b})
  \right] \\ = &\mathbb{E}_{\hat{P}}\left[\tr\left(\mathrm{cov}\left[\hat{b}\right]
  \l(I- \hY\r)^{-\top}M\l(I- \hY\r)^{-1}\right)
  \right],
\end{align*}
and likewise one has
\begin{align*}
  &\mathbb{E}_{\hat{P}}\left[    b^{\top}\l(I- \hY\r)^{-\top}M\l(I- \hY\r)^{-1}b    \right]\\
  = &\mathbb{E}_{\hat{P}}\left[ \tr\l(   b^{\top}b\l(I- \hY\r)^{-\top}M\l(I- \hY\r)^{-1} \r)   \right].
\end{align*}

The entries of $\hat{b}$ are uncorrelated because as defined in \eqref{def bpi},
\[
\mathbb{P}(b_{s_1} = r_{s_1}^{a_1}, b_{s_2} = r_{s_2}^{a_2}) = \pi_{s_1,a_2} \pi_{s_1,a_2} = \mathbb{P}(b_{s_1} = r_{s_1}^{a_1})\mathbb{P}(b_{s_2} = r_{s_2}^{a_2}).
\]
As a result, $\mathrm{cov}\left[\hat{b}\right]$ is a diagonal matrix as claimed.
\end{proof}

\subsection{Proof of Theorem \ref{thm: structure of Y'Y and Y^2} and
derivation of \eqref{eq: vstar}.}
\label{sec: proof of closed-form theorem}

\paragraph{Derivation of \eqref{eq: vstar}.} We first show the derivation of \eqref{eq: vstar}. First one inserts the truncated Neumann series
into the definition of $\v^*$ in \eqref{eq: vstar in Y}. According to \eqref{eq: neumann series},
\begin{equation*}
    \l(I- \hY\r)^{-1} \approx I + \hY + \hY^{2}.
\end{equation*}
One has the following series of approximations by truncating out terms beyond order two
\[
    M\l(I- \hY\r)^{-1} \approx M + M\hY + M\hY^{2},
\]
\[
    \l(I- \hY\r)^{-\top}M\l(I- \hY\r)^{-1} \approx M + M\l(\hY+\hY^2\r) + (\hY^{\top} + (\hY^{\top})^2)M +  \hY^{\top}M\hY.
\]
Note that $\mathbb{E}\left[\h{P}\right] = P$ due to Assumption 1. Thus
$\mathbb{E}\l[\hY\r]=0$. Therefore, taking expectation of the above two terms gives
\begin{equation}\label{eq: numerator in vstar}
    \mathbb{E}_{\hat{P}}\left[M\l(I- \hY\r)^{-1}\right] \approx M + M\mathbb{E}_{\hat{P}}\left[\hY^{2}\right],
\end{equation}
\begin{equation}\label{eq: denominator in vstar}
    \mathbb{E}_{\hat{P}}\left[\l(I- \hY\r)^{-\top}M\l(I- \hY\r)^{-1}\right] \approx M + M\mathbb{E}_{\hat{P}}\left[\hY^2\right] + \mathbb{E}_{\hat{P}}\left[(\hY^{\top})^2\right]M +  \mathbb{E}_{\hat{P}}\left[\hY^{\top}M\hY\right].
\end{equation}
Plugging \eqref{eq: numerator in vstar} and \eqref{eq: denominator in vstar} into \eqref{eq: vstar
  in Y} leads to
\begin{align*}
  \v^{*}    &=
  \frac{
    \mathbb{E}_{\hat{P}}\left[
      b^{\top}M\l(I- \hY\r)^{-1}b\right]
  }{
    \mathbb{E}_{\hat{P}}\left[\tr\left(\l(\mathrm{cov}\left[\hat{b}\right] + b^{\top}b\r)\l(I- \hY\r)^{-\top}M\l(I- \hY\r)^{-1}\right)\right]
  }\\
  &\approx
  \frac{
    \mathbb{E}_{\hat{P}}\left[
      b^{\top} ( M + M\hY^2) b
      \right]
  }{
     \mathbb{E}_{\hat{P}}\left[
      \tr\left(\left(\mathrm{cov}\left[\hat{b}\right] + b^{\top}b\right)
      ( M + \hY^{\top}M\hY + M\hY^{2} + (\hY^{\top})^2M)\right)\right]},
\end{align*}
where the numerator term can be symmetrized so as to get \eqref{eq: vstar}.

\paragraph{Proof of Theorem \ref{thm: structure of Y'Y and Y^2}.}
Let $N = \lv \S \rv$. Denote by $\{\hp_i\}_{i=1}^N$ and $\{p_i\}$ the row vectors of $\h{P}$ and $P$,
respectively:
\begin{equation*}
    \hat{P} = \begin{bmatrix}
    \hp_{1}^{\top}\\
    \dots\\
    \hp_{N}^{\top}
    \end{bmatrix},\quad
    P = \begin{bmatrix}
    p_{1}^{\top}\\
    \dots\\
    p_{N}^{\top}
    \end{bmatrix}.
\end{equation*}

To show that $\v^{\star}$ follows the formula in Theorem \ref{thm: structure of Y'Y and Y^2}, it
suffices to prove the following auxiliary lemma:

\begin{lemma}\label{lemma: auxiliary lemma for YMY and Y^2}
  Assume the following two conditions hold.
  \begin{itemize}
    \item [] {(a)}: \(\h{P},\h{b}\) are unbiased estimators of \(P,b\).
      
    \item [] {(b)}: \(X_{i}\) is independent to \(X_{j}\) whenever \(i \not = j\).
\end{itemize}
  Then one has
  \begin{align}
    &\mathbb{E}\left[\hY^{\top}M\hY\right]
    = \g^2 A^{-\top}\left(\sum_{i=1}^{N}\left[M_{ii}\right]\mathrm{cov}[\hp_{i}]\right)A^{-1},\\
    & \mathbb{E}\left[\hY^2 \right] = \g^2\sum_{i = 1}^{N}\mathrm{diag}\l(e_{i} \r)A^{-\top}\mathrm{cov}[\hp_{i}]A^{-1},\\
    &\mathbb{E}\left[(\hY^{\top})^2 \right]  =  \g^2\sum_{i = 1}^{N}A^{-\top}\mathrm{cov}[\hp_{i}]A^{-1}\mathrm{diag}\l(e_{i} \r),
  \end{align}
  where $\hp_{i}$ is the random vector corresponding to the i-th row of $\h{P}$.
\end{lemma}

Both conditions in Lemma \ref{lemma: auxiliary lemma for YMY and Y^2} are satisfied under Assumption
1. In Theorem \ref{thm: structure of Y'Y and Y^2}, one has $X_{i} \sim
\mathrm{multinomial}(n,p_{i})$ with the following covariance structure
\begin{equation}\label{eq: covaraince of X_i}
  \mathrm{cov}[\hp_{i}] = \frac{1}{n}\left[\diag{p_{i}} - p_{i}p_{i}^{\top}\right] = B_{i}.
\end{equation}
Plugging in \eqref{eq: covaraince of X_i} in Lemma \ref{lemma: auxiliary lemma for YMY and Y^2}
immediately gives the expectation-free form in Lemma \ref{lemma: vstar in Y}, which proves Theorem
\ref{thm: structure of Y'Y and Y^2}.

\begin{proof} (\textbf{Proof of Lemma \ref{lemma: auxiliary lemma for YMY and Y^2}.})
We first calculate $\mathbb{E}\left[\hY^{\top}M\hY\right]$. To do this, we rely on the assumption that the $i$-th row $\hp_{i}$ is independent to $\hp_{j}$ whenever
$i\not=j$. As a consequence, the rows of $\hZ$ are independent. Then, for any matrix $M$, one has
\[
\mathbb{E}\left[\hY^{\top}M\hY\right] =\mathbb{E}\left[A^{-\top}\hZ^{\top}M\hZ A^{-1}\right] =
A^{-\top}\mathbb{E}\left[\hZ^{\top}M\hZ\right]A^{-1}.
\]
By denoting the rows of $\hZ$ by $\hz_{1}^{\top}, \dots, \hz_{N}^{\top}$,
\begin{equation*}
  \hZ^{\top}M\hZ = \begin{bmatrix} \hz_{1} & \dots & \hz_{N}
  \end{bmatrix}M\begin{bmatrix}
  \hz_{1}^{\top}\\
  \dots\\
  \hz_{N}^{\top}
  \end{bmatrix} = \sum_{i,j = 1}^{N}\hz_{i}M_{ij}\hz_{j}^{\top}.
\end{equation*}
By taking the expectation, the only non-zero terms are the ones with $i=j$. Hence,
\begin{equation*}
  \mathbb{E}\left[\hZ^{\top}M\hZ\right] = \sum_{i = 1}^{N}M_{ii}\mathbb{E}\left[\hz_{i}\hz_{i}^{\top}\right].
\end{equation*}
Then by definition of $\hZ$ one has
\begin{equation}\label{eqn: covaraince structure of Z'Z}
  \mathbb{E}\left[\hz_{i}\hz_{i}^{\top}\right] =
  \g^2\mathbb{E}\left[\left(\hp_{i} - \mathbb{E}\left[\hp_{i}\right]\right)\left(\hp_{i} - \mathbb{E}\left[\hp_{i}\right]\right)^{\top}\right] =
  \g^2\mathrm{cov}[\hp_{i}].
\end{equation}
Hence one can get the first part of Lemma \ref{lemma: auxiliary lemma for YMY and Y^2}, which is
\[
\mathbb{E}\left[\hY^{\top}M\hY\right] =
A^{-\top}\mathbb{E}\left[\hZ^{\top} M \hZ\right]A^{-1} = \gamma^2
A^{-\top}\left(\sum_{i = 1}^{N}M_{ii}\mathrm{cov}[\hp_{i}]\right)A^{-1}.
\]

Now we move on to proving the form of $\mathbb{E}\left[\hY^{2}\right]$. Writing out $\hY^{2}$ explicitly
\begin{equation*}
  \hY^{2} = \hZ A^{-1}\hZ A^{-1} = \begin{bmatrix}
    \hz_{1}^{\top}A^{-1}\\
    \dots\\
    \hz_{N}^{\top}A^{-1}
    \end{bmatrix}
    \begin{bmatrix}
    \hz_{1}^{\top}\\
    \dots\\
    \hz_{N}^{\top}
    \end{bmatrix}A^{-1}
    = \sum_{i = 1}^{N}\sum_{j = 1}^{N}\begin{bmatrix}
      0
      \\
    \dots
    \\
    \hz_{i}^{\top}A^{-1}\\
    \dots\\
    0
    \end{bmatrix}
    \begin{bmatrix}
    0
    \\
    \dots
    \\
    \hz_{j}^{\top}\\
    \dots\\
    0
    \end{bmatrix}A^{-1}.
\end{equation*}
After the expectation, the only non-zero terms are $i = j$. Thus one has
\begin{equation*}
    \mathbb{E}\left[\hY^{2} \right]
    =
    \sum_{i = 1}^{N}\mathbb{E}\left[\begin{bmatrix}
    0
    \\
    \dots
    \\
    \hz_{i}^{\top}A^{-1}\\
    \dots\\
    0
    \end{bmatrix}
    \begin{bmatrix}
    0
    \\
    \dots
    \\
    \hz_{i}^{\top}\\
    \dots\\
    0
    \end{bmatrix}\right]A^{-1},
\end{equation*}
with
\begin{equation*}
    \left(\begin{bmatrix}
    0
    \\
    \dots
    \\
    \hz_{i}^{\top}A^{-1}\\
    \dots\\
    0
    \end{bmatrix}
    \begin{bmatrix}
    0
    \\
    \dots
    \\
    \hz_{i}^{\top}\\
    \dots\\
    0
    \end{bmatrix}\right)_{jk} =
    \begin{cases}
        \sum_{l = 1}^{N}{\hZ_{il}A^{-1}_{li} \hZ_{ik}}  & j = i\\
        0 & j \not = i.
    \end{cases}
\end{equation*}

For the matrix $A^{-\top}\hz_{i}\hz_{i}^{\top}$, note that
\begin{equation*}
  \left[A^{-\top}\hz_{i}\hz_{i}^{\top}\right]_{jk} = (A^{-\top}\hz_{i})_{j}(\hz_{i}^{\top})_{k} =
  \sum_{l=1}^{N}A^{-\top}_{jl}\hZ_{il}\hZ_{ik} = \sum_{l=1}^{N}A^{-1}_{lj}\hZ_{il}\hZ_{ik}.
\end{equation*}
Applying $j = i$ leads to
\begin{equation*}
  \left[\mathrm{diag}\l(e_{i} \r)A^{-\top}\hz_{i}\hz_{i}^{\top}\right]_{jk} = 
  \begin{cases}
    \sum_{l=1}^{N}\hZ_{il}A^{-1}_{li}\hZ_{ik} & j = i \\
    0 & j \not = i
  \end{cases}.
\end{equation*}
Hence we have 
\begin{align*}
  \mathbb{E}\left[\hY^{2} \right]
  &
  = 
  \sum_{i = 1}^{N}\mathbb{E}\left[\begin{bmatrix}
      0
      \\
      \dots
      \\
      \hz_{i}^{\top}A^{-1}\\
      \dots\\
      0
    \end{bmatrix}
    \begin{bmatrix}
      0
      \\
      \dots
      \\
      \hz_{i}^{\top}\\
      \dots\\
      0
    \end{bmatrix}\right]A^{-1} = 
  \sum_{i = 1}^{N}\mathbb{E}\left[\mathrm{diag}\l(e_{i} \r)\begin{bmatrix}
      0
      \\
      \dots
      \\
      \hz_{i}^{\top}A^{-1}\\
      \dots\\
      0
    \end{bmatrix}
    \begin{bmatrix}
      0
      \\
      \dots
      \\
      \hz_{i}^{\top}\\
      \dots\\
      0
    \end{bmatrix}\right]A^{-1} \\
  &
  = 
  \sum_{i = 1}^{N}\mathbb{E}\left[\mathrm{diag}\l(e_{i} \r)A^{-\top}\hz_{i}\hz_{i}^{\top}\right]A^{-1}
  = 
  \sum_{i = 1}^{N}\mathrm{diag}\l(e_{i} \r)A^{-\top}\mathbb{E}\left[\hz_{i}\hz_{i}^{\top}\right]A^{-1}.
\end{align*}
Taking transpose results in
\begin{equation*}
  \mathbb{E}\left[(\hY^{\top})^2 \right] = \sum_{i = 1}^{N}A^{-\top}\mathbb{E}\left[\hz_{i}\hz_{i}^{\top}\right]A^{-1}\mathrm{diag}\l(e_{i} \r).
\end{equation*}
\end{proof}

\subsection{Proof of Lemma \ref{lemma: structure of objective function} and Corollary \ref{corollary: structure of error reduction}.}\label{sec: proof of corollary on error asymptotics}

We first prove  Lemma \ref{lemma: structure of objective function}.
\begin{proof}

(\textbf{Proof of Lemma \ref{lemma: structure of objective function}.})

  Going back to the original quadratic optimization problem, one has
  \begin{equation}\label{eqn: objective function for epsilon}
    \E_{\hat{P},\hat{b}}\ll b- \varepsilon A \hat{v} \rl_{M}^{2} = \varepsilon^2\E_{\hat{P},\hat{b}}\ll A \hat{v} \rl_{M}^{2} - 2\varepsilon\E_{\hat{P},\hat{b}}\left[b^{\top}MA \hat{v}\right] + \ll b \rl_{M}^{2}.
  \end{equation} 
  Using Lemma \ref{lemma: vstar in Y} and the second-order approximation in equation \eqref{eq:
    vstar}, one has
  \begin{equation*}
    \begin{aligned}
      \E_{\hat{P},\hat{b}}\left[b^{\top}MA \hat{v}\right]
      &= \mathbb{E}_{\hat{P}}\left[
        b^{\top} ( M + \frac{M\hY^2+ (\hY^{\top})^{2}M}{2}) b
        \right] + h.o.t.,
    \end{aligned}
  \end{equation*} 
  and
  \begin{equation*}
    \begin{aligned}
      \E_{\hat{P},\hat{b}}\ll A \hat{v} \rl_{M}^{2} 
      &=  \mathbb{E}_{\hat{P}}\left[
        b^{\top} ( M + \hY^{\top}M\hY + M\hY^{2} + (\hY^{\top})^2M) b\right]\\
      &+ \mathbb{E}_{\hat{P}}\left[\tr\left(\mathrm{cov}\left[\hat{b}\right]
        ( M + \hY^{\top}M\hY + M\hY^{2} + (\hY^{\top})^2M)\right)\right] + h.o.t.,
    \end{aligned}
  \end{equation*}
  where $h.o.t.$ stands for high order terms.

  We now show that $h.o.t. = O\l(n^{-\frac{3}{2}}\r)$. The expectation of the third order or higher
  terms in $\hY$ is computed by moments of third order or
  higher in $\hZ$. Under Assumption 1, rows of $\hZ$ are independent, and
  hence moments of $\hZ$ are linear combinations of moments in multinomial
  distribution. Each row of matrix $\hZ$ is an average of $n$
  random variables with mean zero, which is why its moments of third order or higher decay at the
  rate of at least $O\l(n^{-\frac{3}{2}}\r)$ by use of the Marcinkiewicz-Zygmund inequality.
  
  One can then plug in the explicit formula from Lemma \ref{lemma: auxiliary lemma for YMY and Y^2}
  and \eqref{eq: covaraince of X_i}, leading to
  \begin{equation}\label{eq: first order term}
    \E_{\hat{P},\hat{b}}\left[b^{\top}MA \hat{v}\right] = \ll b \rl_{M}^{2} + b^{\top}Hb/2 + O\l(n^{-\frac{3}{2}}\r)= \ll b \rl_{M}^{2} + h/2 +O\l(n^{-\frac{3}{2}}\r).
  \end{equation}
  and
    \begin{equation}\label{eq: second order term}
    \begin{aligned}
        \E_{\hat{P},\hat{b}}\ll A \hat{v} \rl_{M}^{2} &= \ll b \rl_{M}^{2} + b^{\top}Gb + b^{\top}Hb + \tr\left(\mathrm{cov}\left[\hat{b}\right]
        \l( M + G + H\r)\right) + O\l(n^{-\frac{3}{2}}\r)\\
        &= \ll b \rl_{M}^{2} + h + g + t+ O\l(n^{-\frac{3}{2}}\r).
    \end{aligned}
    \end{equation}
    Plugging it into \eqref{eqn: objective function for epsilon} results in \eqref{eq: objective
      under Taylor second-order}.
    
    We now move to prove \eqref{eq: v subtracts vstar}.  From \eqref{eqn: objective function for
      epsilon} and the results before, one has
    \begin{equation*}
        \v^{*} =  \frac{\ll b \rl_{M}^{2} + h/2+ O\l(n^{-\frac{3}{2}}\r)}{\ll b \rl_{M}^{2} + h + g + t+ O\l(n^{-\frac{3}{2}}\r)}.
    \end{equation*}
    
    On the other hand, Theorem \ref{thm: structure of Y'Y and Y^2} proves that \(\v^{\circ}\) follows the following form:
    \begin{equation}\label{eq: vstar in g h t}
        \v^{\star} =  \frac{\ll b \rl_{M}^{2} + h/2}{\ll b \rl_{M}^{2} + h + g + t}.
    \end{equation}
    
    Moreover, from Theorem \ref{thm: structure of Y'Y and Y^2}, it follows that $G,H,\mathrm{cov}\left[\hat{b}\right] \propto \frac{1}{n}$, and therefore $g,h,t \propto \frac{1}{n}$. Hence \eqref{eq: v subtracts vstar} holds and $\v^{\star}$ is asymptotically optimal.

\end{proof}


\begin{proof}
(\textbf{Proof of Corollary \ref{corollary: structure of error reduction}.})



Throughout this proof, we use the fact that $g,h,t \propto \frac{1}{n}$ as in the proof of Lemma \ref{lemma: structure of objective function}. Without loss of generality, assume that $\ll b \rl_{M}^{2} = 1$.
The proof is organized as follows. First, we prove that \(\mathrm{MSE}(1) - \mathrm{MSE}(\v^{*}) =O\l(n^{-2}\r)\). Second, we prove that \(\mathrm{MSE}(\v^{*}) - \mathrm{MSE}(\v^{\circ}) =O\l(n^{-3}\r)\). As a consequence, one obtains \(\mathrm{MSE}(1) - \mathrm{MSE}(\v^{\circ}) =O\l(n^{-2}\r)\). Then, note that
\begin{equation}
    \mathrm{MSE}(1) =  (g+t) + O\l(n^{-\frac{3}{2}}\r) = O\l(1/n \r), 
\end{equation}
and so the relative error is of order \(O\l(1/n \r)\) as claimed, and \(\eta\) is positive for \(n\) sufficiently large.

We first estimate \(\mathrm{MSE}(1) - \mathrm{MSE}(\v^{*})\).
\tcb{Plugging in \eqref{eqn: eqn1 for operator shifting}, one has}
\begin{equation}\label{eq: order of oracle plug-in error}
    \mathrm{MSE}(1) - \mathrm{MSE}(\v^{*}) = 
    \frac{\left(\ll A \hat{v} \rl_{M}^{2} - \E_{\hat{P},\hat{b}}\left[b^{\top}MA \hat{v}\right]\right)^2}{\ll A \hat{v} \rl_{M}^{2}}.
\end{equation}

Under the assumption that $\ll b \rl_{M}^{2} = 1$, \eqref{eq: first order term} and \eqref{eq: second order term} shows
\begin{equation*}
    \E_{\hat{P},\hat{b}}\ll A \hat{v} \rl_{M}^{2} = 1 + h + g + t+ O\l(n^{-\frac{3}{2}}\r), \E_{\hat{P},\hat{b}}\left[b^{\top}MA \hat{v}\right] = 1 + h/2 +O\l(n^{-\frac{3}{2}}\r).
\end{equation*}

Thus, plugging in the previously derived terms into \eqref{eq: order of plug-in error}, one has
\begin{equation}
    \mathrm{MSE}(1) - \mathrm{MSE}(\v^{*}) =  \frac{\l(g+h/2+t+O\l(n^{-\frac{3}{2}}\r)\r)^2}{1+g+h+t+ O\l(n^{-\frac{3}{2}}\r)} = O\l(n^{-2}\r).
\end{equation}

We then estimate \(\mathrm{MSE}(\v^{*}) - \mathrm{MSE}(\v^{\circ})\):
\begin{equation}\label{eq: order of plug-in error}
    \mathrm{MSE}(\v^{*}) - \mathrm{MSE}(\v^{\circ}) =  (\left(\v^{*}\right)^2-\left(\v^{\circ}\right)^2)\E_{\hat{P},\hat{b}}\ll A \hat{v} \rl_{M}^{2} - 2(\v^{*}-\v^{\circ})\E_{\hat{P},\hat{b}}\left[b^{\top}MA \hat{v}\right].
\end{equation}

Then, one uses
\[
\v^{*} = \frac{\E_{\hat{P},\hat{b}}\left[b^{\top}MA \hat{v}\right]}{\E_{\hat{P},\hat{b}}\ll A \hat{v} \rl_{M}^{2}}, \quad  \v^{\circ} = \frac{1+h/2 }{1+h+g+t}.
\]

By simple algebra, one obtains from \eqref{eq: order of plug-in error} that
\begin{equation}\label{eq: order of plug-in error step 2}
    \mathrm{MSE}(\v^{*}) - \mathrm{MSE}(\v^{\circ}) =  -\frac{
    \l(
    (1+h+g+t)\E_{\hat{P},\hat{b}}\left[b^{\top}MA \hat{v}\right] - \E_{\hat{P},\hat{b}}\ll A \hat{v} \rl_{M}^{2}(1+h/2)
    \r)^2
    }{\l(\E_{\hat{P},\hat{b}}\ll A \hat{v} \rl_{M}^{2}\r)(1+h+g+t)^2},
\end{equation}

Importantly, for the numerator term in \eqref{eq: order of plug-in error step 2}, note that \[(1+h+g+t)\E_{\hat{P},\hat{b}}\left[b^{\top}MA \hat{v}\right] = (1+h+g+t)(1+h/2) + O\l(n^{-\frac{3}{2}}\r),\]
and likewise one has
\[
\E_{\hat{P},\hat{b}}\ll A \hat{v} \rl_{M}^{2}(1+h/2) = (1+h+g+t)(1+h/2) + O\l(n^{-\frac{3}{2}}\r).
\]

Consequently, the numerator term in \eqref{eq: order of plug-in error step 2} is of order \(O\l(n^{-3}\r)\).
Thus
\begin{equation}
    \mathrm{MSE}(\v^{*}) - \mathrm{MSE}(\v^{\circ}) =  \frac{O\l(n^{-3}\r)}{\l(1+g+h+t+ O\l(n^{-\frac{3}{2}}\r)\r)\l(1+g+h+t\r)^2} = O\l(n^{-3}\r),
\end{equation}
as is desired.
\end{proof}

\subsection{Proof of Theorem \ref{thm: upper bd}.}
To prove Theorem \ref{thm: upper bd}, one first finds a tight bound for $A^{-1}b$ and
$A^{-1}\mathrm{diag}(e_i)b$. The tight upper and lower bounds for $A^{-1}b$ are stated in Lemma
\ref{lemma: bd for x}. Then, the upper bounds for $\ll A^{-1} b \rl^2$ and $\sum_i\ll
A^{-1}\mathrm{diag}(e_i)b \rl^2$ are listed in Corollary \ref{coro: ineq}. Finally, based on
Corollary \ref{coro: ineq}, we derive the bound for $\v^\star$ in Theorem \ref{thm: upper bd}.
\begin{lemma}\label{lemma: bd for x}
For any transition matrix $P\in\R^{\lv \S \rv\times\lv \S \rv}$, vector $b\in\R^{\lv \S \rv}$ and $\g \in(0,1)$, 
\[
     b + \frac{\g}{1-\g} b_m{\bf 1}\leq \l( I-\g P \r)^{-1} b \leq  b + \frac{\g}{1-\g}  b_M{\bf 1},
\]
where $b_m = \min_s b_s, b_M = \max_s b_s$, \tcb{and  inequality between vectors denotes an entry-wise inequality}.
\end{lemma}
\begin{proof}
Let $x = \l( I-\g P \r)^{-1}b$, and $s = \argmin_i x_i$, then the $s$-th row of $\l( I-\g P \r)x = b$ is
\[
    b_s = x_s - \g \sum_t P_{st}x_t \leq x_s - \g \sum_t P_{st}x_s = \l(1-\g\r)x_s, 
\]
which implies
\begin{equation}\label{eq: pf_1}
    x_s \geq \frac{b_s}{1-\g} \geq \frac{b_m}{1-\g},
\end{equation}
where $b_m = \min_s b_s$. For $\forall j \neq s$, one has
\[
b_j = x_j - \g \sum_t P_{jt}x_t \leq x_j - \g \sum_t P_{st}x_s = x_j - \g x_s \leq x_j - \g \frac{b_m}{1-\g}, 
\]
which yields,
\begin{equation}\label{eq: pf_2}
    x_j \geq b_j + \frac{\g}{1-\g} b_m.
\end{equation}
Combining \eqref{eq: pf_1} and \eqref{eq: pf_2} gives
\[
x = \l( I-\g P \r)^{-1}b \geq b + \frac{\g}{1-\g}b_m{\bf 1}.
\]

On the other hand, let $l = \argmax_i x_i$, then the $l$-th row of $\l( I-\g P \r)x = b$ is
\[
b_l = x_l - \g\sum_t P_{lt}x_t \geq x_l - \g \sum_t P_{lt}x_l = \l(1-\g\r)x_l,
\]
which implies
\begin{equation}\label{eq: pf_3}
    x_l \leq \frac{b_l}{1-\g} \leq \frac{b_M}{1-\g},
\end{equation}
where $b_M = \max_s b_s$. For $\forall j \neq l$, one has
\[
    b_j = x_j - \g\sum_tP_{jt}x_t \geq x_j - \g \sum_tP_{jt}x_l \geq x_j - \g \frac{b_M}{1-\g},
\]
which yields,
\begin{equation}\label{eq: pf_4}
    x_j \leq b_j + \frac{\g}{1-\g} b_M.
\end{equation}
Combining \eqref{eq: pf_3} and \eqref{eq: pf_4} gives
\[
x = \l( I-\g P \r)^{-1}b \leq b + \frac{\g}{1-\g} b_M{\bf 1},
\]
which completes the proof. 
\end{proof}

\begin{bluetext}

\begin{lemma}\label{lemma: abs x}
\tcb{For a vector \(v = (v_{i})_{i= 1}^{d}\), define \(\lv v \rv := (\lv v_{i} \rv)_{i= 1}^{d}\), i.e. the entry-wise absolute value of \(v\).} Suppose a matrix $Q$ has only non-negative entries. Then, for any vector $v\in\R^{\lv \S \rv}$ and $\g\in(0,1)$, 
\[
\lv Q v \rv \leq Q \lv v \rv .
\]
\end{lemma}
\begin{proof}
Define $x =  Q v$. Denote by $v_+, v_-$ the positive and negative parts of $v$, respectively. That is, $(v_+)_i = v_i\mathds{1}_{v_i > 0}$ and $(v_-)_i = v_i\mathds{1}_{v_i \leq 0}$. Since \(Q\) has only non-negative entries, it follows \(Q v_+ \geq {\bf 0}\) and \(Q v_- \leq {\bf 0}\). 
Because $x = Q v_+ + Q v_-$, one has
$x \leq Q v_+ - Q v_-$ and 
$-x \geq -Q v_+ + Q v_-$.
Hence
\[
\lv x\rv \leq Q v_+ - Q v_- = Q  \lv v \rv.
\]


\end{proof}
\end{bluetext}

\begin{corollary}\label{coro: ineq}
For any transition matrix $P$, vector $b\in\R^{\lv \S \rv}$ and $\g\in(0,1)$, one has, 
\[
\begin{aligned}
&\ll \l( I-\g P \r)^{-1}b \rl_2 \leq \ll k\rl_2, \\& \sum_i \ll \l( I-\g P \r)^{-1}\diag(e_i)b \rl^2_2 \leq \ll k\rl^2_2,
\\
&\sum_i \ll \l( I-\g P \r)^{-1}\diag(e_i)b \rl_2 \leq \sqrt{\lv \S \rv}\ll k\rl_2,
\end{aligned}
\]
where $k = \frac{1}{1-\g}\l( \l(1-\g\r)\lv b \rv + \g b_M {\bf 1}\r)$ with $b_M = \max_i |b_i|$ and
\[
    \ll k \rl^2_2 = \frac1{\l(1-\g\r)^2}\l(\l(1-\g\r)^2\ll b\rl_2^2 +\lv \S \rv\g^2b_M^2 + 2\g\l(1-\g\r)b_M\ll b\rl_1\r).
\]
\end{corollary}

\begin{proof}
Note that \(\l( I-\g P \r)^{-1}\) is a matrix with non-negative entries, and therefore one has
\[
\lv \l( I-\g P \r)^{-1} b \rv \leq \l( I-\g P \r)^{-1} \lv b \rv.
\]
Then, Lemmas \ref{lemma: bd for x} leads to
\[
\lv \l( I-\g P \r)^{-1}b  \rv \leq \l( I-\g P \r)^{-1}\lv b \rv  \leq  k = \frac{1}{1-\g}\l( \l(1-\g\r)\lv b \rv + \g b_M {\bf 1}\r),
\]
where $b_M = \max_i |b_i|$, which is the first inequality in the corollary. 

The second inequality is
because
\[
\sum_i \ll \l( I-\g P \r)^{-1}\text{diag}(e_i)b \rl^2_2 \leq \ll \sum_i\lv \l( I-\g P
\r)^{-1}\text{diag}(e_i)b \rv \rl_2^2 \leq \ll \l( I-\g P \r)^{-1}\lv b\rv \rl^2_2 \leq \ll k
\rl^2_2.
\]

The third inequality is due to
\begin{align*}
&\sum_i \ll \l( I-\g P \r)^{-1}\text{diag}(e_i)b \rl_2 \\\leq & \sqrt{\lv \S \rv} \ll \sum_i\lv \l( I-\g
P \r)^{-1}\text{diag}(e_i)b \rv \rl_2 \\\leq & \sqrt{\lv \S \rv} \ll \l( I-\g P \r)^{-1}\lv b\rv \rl_2
\\\leq & \sqrt{\lv \S \rv}\ll k \rl_2.
\end{align*}

\end{proof}

Now we are ready to proof Theorem \ref{thm: upper bd}. 

\begin{proof} (\textbf{Proof of Theorem \ref{thm: upper bd}.})

Since $B_i$ is the covariance matrix, $B_i \succeq 0$. By the definition of $G,H$ in Theorem
\ref{thm: structure of Y'Y and Y^2}, one has $\tr\left(\mathrm{cov}\left[\hat{b}\right]
\l(I+G+H\r)\right)\geq 0$. By letting $a_i = \l( I-\g P \r)^{-1}\text{diag}(e_i)b$ and $d = \l( I-\g
P \r)^{-1}b$, 
\begin{equation}\label{eq: pf_6}
\begin{aligned}
  \v^\star \leq & \frac{\frac{1}{\g^2}b^\top b + \sum_i a_i^\top B_i d}{\frac{1}{\g^2}b^\top b + \sum_id^\top B_i d + 2\sum_i a_i^\top B_i d} = 1 - \frac{\sum_id^\top B_i d + \sum_i a_i^\top B_i d}{\frac{1}{\g^2}b^\top b + \sum_id^\top B_i d + 2\sum_i a_i^\top B_i d} \\
  =&1+\frac{\frac12\sum_i a_i^\top B_i a_i - \frac12\sum_i(a_i+d)^\top B_i (a_i+d) - \frac12\sum_id^\top B_i d}{\frac{1}{\g^2}b^\top b + \sum_i(a_i+d)^\top B_i (a_i+d) - \sum_i a_i^\top B_i a_i}
\end{aligned}
\end{equation}
The numerator of the second term in \eqref{eq: pf_6} can be bounded by
\[
\frac12\sum_i a_i^\top B_i a_i - \frac12\sum_i(a_i+d)^\top B_i (a_i+d) - \frac12\sum_id^\top B_i d\leq \frac12\lam_M\sum_i\ll a_i \rl_2^2 ,
\]
where $\lam_M$ is the largest eigenvalue of $B_i$ for $\forall i$.

Suppose that $\frac{\lam_M\ll k\rl^2_2}{\frac{1}{\g^2}b^\top b} \leq \frac12$. Then the denominator
of the second term of \eqref{eq: pf_6} can be lower bounded by
\begin{equation}\label{eq: bound on I + G+ H}
\frac{1}{\g^2}b^\top b + \sum_i(a_i+d)^\top B_i (a_i+d) - \sum_i a_i^\top B_i a_i \geq \frac{1}{\g^2}b^\top b - \lam_M\ll k\rl^2_2 > 0,
\end{equation}

where  $\sum_i \ll a_i \rl^2_2 \leq \ll k \rl^2_2$ from Corollary \ref{coro: ineq} is used. 
Therefore, \eqref{eq: pf_6} can be bounded by
\[
\v^\star \leq  1+ \frac{\frac12\lam_M \ll k\rl_2^2}{\frac{1}{\g^2}b^\top b - \lam_M\ll k\rl^2_2}
\leq 1 + \frac{\lam_M\ll k\rl^2_2}{\frac{1}{\g^2}b^\top b}.
\]
Note that
\begin{equation}\label{eq: pf_5}
    \begin{aligned}
\frac{\ll k \rl^2_2}{\frac{1}{\g^2}\ll b \rl^2_2} \leq&  \frac{\g^2}{\l(1-\g\r)^2}\frac{\l(\l(1-\g\r)^2\ll b\rl_2^2 + \lv \S \rv\g^2b_M^2 + 2\g\l(1-\g\r)b_M\ll b\rl_1\r)}{\ll b \rl^2_2} \\
\leq& \frac{\g^2}{\l(1-\g\r)^2}\l(\l(1-\g\r)^2 + \lv \S \rv\g^2\frac{b_M^2}{\ll b\rl_2^2} + 2\g\l(1-\g\r)\sqrt{\lv \S \rv}\frac{b_M}{\ll b \rl_2}\r) \\=& \frac{\g^2}{\l(1-\g\r)^2} \l(\l(1-\g\r) + \g\frac{\sqrt{\lv \S \rv}b_M}{\ll b \rl_2}\r)^2.
\end{aligned}
\end{equation}
where $\ll b \rl_1 \leq \sqrt{\lv \S \rv}\ll b \rl_2$ is used in the above inequality.  The largest
eigenvalue $\lam_M$ of $B_i$ defined in \eqref{eq: def of G and H} is smaller than $\lam_M <
\frac{p_M}{n}$, where $p_M = \max_{i,j} P_{i,j}$ is the maximum probability of the transition matrix
$P$ \cite{10.2307/2346179}. This implies that
\[
\frac{\lam_M\ll k \rl^2_2}{\frac{1}{\g^2}\ll b \rl^2_2} \leq \frac{p_M}{n} \frac{\g^2}{\l(1-\g\r)^2}
\l(\l(1-\g\r) + \g\frac{\sqrt{\lv \S \rv} b_M}{\ll b \rl_2}\r)^2 \le \frac12,
\]
\tcb{where the second inequality follows from the stronger assumption that $\frac{p_M}{n}\frac{\g^2}{\l(1-\g\r)^2}\l(\frac{\l(1-\g\r)}{\g} + \frac{\sqrt{\lv \S \rv} b_M}{\ll b \rl_2}\r)^2 \leq \frac{1}{2}$}. Therefore,
\[
\v^\star 
\leq   1+ \frac{p_M}{n} \frac{\g^2}{\l(1-\g\r)^2} \l(\l(1-\g\r) \g\frac{\sqrt{n} b_M}{\ll b \rl_2}\r)^2,
\]
which completes the proof for the upper bound of $\v^\star$.

\begin{bluetext}
We now move on to the condition for which \(\epsilon^{\star}>0\). From Theorem \ref{thm: structure of Y'Y and Y^2}, it suffices to show \(b^{\top}b + b^{\top}Hb/2 > 0\) and \(b^{\top}b + b^{\top}Hb +b^{\top}Gb > 0\) to obtain \(\epsilon^{\star} > 0\). As a consequence of \eqref{eq: bound on I + G+ H}, one has
\begin{align*}
&b^\top b + b^{\top}Hb +b^{\top}Gb \\=&b^\top b + \g^2\sum_i(a_i+d)^\top B_i (a_i+d) - \g^2\sum_i a_i^\top B_i a_i\\ 
\geq 
&b^\top b - \g^2\sum_i a_i^\top B_i a_i\\
\geq &b^\top b - \g^2\lam_M\ll k\rl^2_2.
\end{align*}

Hence the denominator term in \(\epsilon^{\star}\) is positive whenever \(b^\top b  > \g^2\lam_M\ll k\rl^2_2\), which holds when \(n > \g^2\lam_M\frac{\ll k\rl^2_2}{ \ll b\rl_{2}^2}\).

For \(b^{\top}b + b^{\top}Hb/2\), one has 
\begin{align*}
    &b^{\top}b + b^{\top}Hb/2
    = b^\top b + \g^2\sum_i a_i^\top B_i d 
    \geq b^\top b - \g^2 \ll \sum_i B_i a_i  \rl_{1} \ll d \rl_{\infty},
\end{align*}
where the inequality follows from Holder's inequality. Moreover, one has \(\ll d \rl_{\infty}  = \ll\l( I-\g
P \r)^{-1}b \rl_{\infty} \leq \ll\l( I-\g
P \r)^{-1} \rl_{\infty} \ll b \rl_{\infty} = \frac{b_M}{1-\gamma}\). Thus, it suffices to bound the term \(\ll \sum_i B_i a_i  \rl_{1}\). We will show that 
\begin{equation}\label{eq: bound on B_iai WTS}
\ll \sum_i B_{i} a_i  \rl_{1} \leq 2\frac{p_{M}}{n}\lv \S \rv\frac{1}{1-\g} b_{M}.
\end{equation}
Assuming \eqref{eq: bound on B_iai WTS} holds, one has
\[
    b^{\top}b + b^{\top}Hb/2 
    \geq b^\top b - \g^2 \ll \sum_i B_i a_i  \rl_{1} \ll d \rl_{\infty}  
    \geq b^\top b - 2\g^{2}\frac{1}{n}p_{M}\lv \S\rv\frac{1}{(1-\gamma)^2}b_{M}^2,
\]
which implies the numerator term \(b^{\top}b + b^{\top}Hb/2\) is positive whenever 
\[
n > 2\g^{2}p_{M}\lv \S\rv\frac{1}{(1-\gamma)^2}\frac{b_{M}^2}{\ll b \rl_{2}^2}.
\]

Thus, for \(\epsilon^{\star} > 0\), one needs
\[
n > \max{\l(\g^2\lam_M\frac{\ll k\rl^2_2}{ \ll b\rl_{2}^2}, 2\g^{2}p_{M}\lv \S\rv\frac{1}{(1-\gamma)^2}\frac{b_{M}^2}{\ll b \rl_{2}^2}\r)}.
\]
By \eqref{eq: pf_5}, one can see that \( n > \g^2\lam_M\frac{\ll k\rl^2_2}{ \ll b\rl_{2}^2}\) holds if $\frac{p_M}{n}\frac{\g^2}{\l(1-\g\r)^2}\l(\l(1-\g\r) + \g\frac{\sqrt{\lv \S \rv} b_M}{\ll b
\rl_2}\r)^2 \leq \frac{1}{2}$. Moreover, the condition that \(n >  2\g^{2}p_{M}\lv \S\rv\frac{1}{(1-\gamma)^2}\frac{b_{M}^2}{\ll b \rl_{2}^2}\) is simply a restatement of the condition that $\frac{p_M}{n}\frac{\g^{2}}{(1-\gamma)^2}\l(\frac{\sqrt{\lv \S\rv}b_{M}}{\ll b \rl_{2}}\r)^2 < \frac{1}{2}$.

Then, since \(\g \in (0,1)\), the theorem's assumption $\frac{p_M}{n}\frac{\g^2}{\l(1-\g\r)^2}\l(\frac{\l(1-\g\r)}{\g} + \frac{\sqrt{\lv \S \rv} b_M}{\ll b \rl_2}\r)^2 \leq \frac{1}{2}$ implies
$\frac{p_M}{n}\frac{\g^2}{\l(1-\g\r)^2}\l(\l(1-\g\r) + \g\frac{\sqrt{\lv \S \rv} b_M}{\ll b \rl_2}\r)^2 \leq \frac{1}{2}$ and $\frac{p_M}{n}\frac{\g^{2}}{(1-\gamma)^2}\l(\frac{\sqrt{\lv \S\rv}b_{M}}{\ll b \rl_{2}}\r)^2 < \frac{1}{2}$. Therefore, the theorem's assumption implies \(\epsilon^{\star} > 0\).

For the remainder of the proof, we show that \eqref{eq: bound on B_iai WTS} holds.
 Define \(B_{i,+} = \frac{1}{n}\diag{p_{i}}\) and \(B_{i,-} =  -\frac{1}{n}p_{i}p_{i}^{\top}\). It follows that \(B_{i} = B_{i,+} + B_{i,-}\), and one has
\begin{equation}\label{eqn: bound on B_iai}
\ll \sum_i B_i a_i  \rl_{1} \leq \ll \sum_i B_{i,+} a_i  \rl_{1} + \ll \sum_i B_{i,-} a_i  \rl_{1}.
\end{equation}

We first bound the first term on the right hand side of \eqref{eqn: bound on B_iai}. Let $b_+, b_-$ be the positive and negative parts of $b$ as in Lemma \ref{lemma: abs x}, and define \[a_{i,+} := \l( I-\g P \r)^{-1}\text{diag}(e_i)b_{+},\quad a_{i,-} := \l( I-\g P \r)^{-1}\text{diag}(e_i)b_{-}.\]
In particular, one has \(a_{i,+} \geq {\bf 0}\) and \(a_{i,+} \leq {\bf 0}\), as \(\l( I-\g P \r)^{-1}\text{diag}(e_i)\) has only non-negative entries.
Thus one can further bound by 
\begin{equation*}
\ll \sum_i B_{i,+} a_i  \rl_{1} \leq \ll \sum_i B_{i,+} a_{i,+}  \rl_{1} + \ll \sum_i B_{i,+} a_{i,-}  \rl_{1}.
\end{equation*}

Due to the non-negativity of the entries in \((-a_{i,-})\) and \(a_{i,+}\), the right hand side is a monotonically non-decreasing function in the entries of \(B_{i,+}\), and therefore one has 
\begin{align*}
&\ll \sum_i B_{i,+} a_{i,+}  \rl_1 + \ll \sum_i B_{i,+} a_{i,-}  \rl_1 \\
\leq &\ll \frac{1}{n}\diag{\l(p_{M} {\bf 1}\r)} \sum_i a_{i,+}  \rl_1 + \ll  \frac{1}{n}\diag{\l(p_{M} {\bf 1}\r)} \sum_i a_{i,-}  \rl_1 \\
= &\frac{p_{M}}{n}\l(\ll \l( I-\g P \r)^{-1}b_{+}  \rl_1 + \ll \l( I-\g P \r)^{-1}b_{-}  \rl_1 \r).
\end{align*}
Now, note that \(\ll \l( I-\g P \r)^{-1}b_{+}  \rl_1 = {\bf 1}^{\top}\l( I-\g P \r)^{-1}b_{+} \), and \(\ll \l( I-\g P \r)^{-1}b_{-}  \rl_1 = -{\bf 1}^{\top}\l( I-\g P \r)^{-1}b_{-} \). Therefore, \[\ll \l( I-\g P \r)^{-1}\lv b\rv   \rl_1 = {\bf 1}^{\top}\l( I-\g P \r)^{-1}\lv b\rv = {\bf 1}^{\top}\l( I-\g P \r)^{-1}\l(b_{+} - b_{-}\r).\]

Hence one has
\begin{align*}
&\ll \sum_i B_{i,+} a_{i,+}  \rl_1 + \ll \sum_i B_{i,+} a_{i,-}  \rl_1 \\
\leq &\frac{p_{M}}{n}\l(\ll \l( I-\g P \r)^{-1}b_{+}  \rl_1 + \ll \l( I-\g P \r)^{-1}b_{-}  \rl_1 \r)\\
    = &\frac{p_{M}}{n}\ll \l( I-\g P \r)^{-1}\lv b \rv  \rl_1\\
\leq &\frac{p_{M}}{n} \lv \S \rv\ll \l( I-\g P \r)^{-1}\lv b \rv\rl_{\infty}\\
\leq &\frac{p_{M}}{n}\lv \S \rv\frac{1}{1-\g} b_{M}.
\end{align*}

For the second term in \eqref{eqn: bound on B_iai}, one similarly has
\[ \ll \sum_i B_{i,-} a_{i}  \rl_1 \leq \ll \sum_i B_{i,-} a_{i,-}  \rl_1 +  \ll \sum_i B_{i,-} a_{i,+}  \rl_1.\]

One can check 
\begin{align*}
&\ll \sum_i B_{i,-} a_{i,-}  \rl_1 +  \ll \sum_i B_{i,-} a_{i,+}  \rl_1 \\= &{\bf 1}^{\top}\sum_{i}\frac{1}{n}p_{i}p_{i}^{\top}\l( I-\g P \r)^{-1}\text{diag}(e_i)\lv b  \rv\\
=
&\sum_{i}\frac{1}{n}p_{i}^{\top}\l( I-\g P \r)^{-1}\text{diag}(e_i)\lv b  \rv\\
\leq 
&\sum_{i}\frac{p_{M}}{n}{\bf 1}^{\top}\l( I-\g P \r)^{-1}\text{diag}(e_i)\lv b  \rv
\\
=
&\frac{p_{M}}{n}{\bf 1}^{\top}\l( I-\g P \r)^{-1}\lv b  \rv\\
\leq 
&\frac{p_{M}}{n} \ll \l( I-\g P \r)^{-1}\lv b  \rv \rl_{1}\\
\leq &\frac{p_{M}}{n}\lv \S\rv\frac{1}{1-\gamma}b_{M}.
\end{align*}
Thus \(\ll \sum_i B_i a_i  \rl_{1} \leq 2\frac{p_{M}}{n}\lv \S\rv\frac{1}{1-\gamma}b_{M}\).

\end{bluetext}

\end{proof}

\bibliographystyle{plain}

\newpage

\appendix
\section*{Appendices}
\addcontentsline{toc}{section}{Appendices}
\renewcommand{\thesubsection}{\Alph{subsection}}
\subsection{Condition for Convergence of Neumann Series.}\label{appendix: condition for convergence of Neumann Series}

\setcounter{equation}{0}
\setcounter{theorem}{0}
\renewcommand\theequation{A.\arabic{equation}}
\renewcommand\thetheorem{A.\arabic{theorem}}

The spectral radius $\rho(\hY)$ of $\hY$ can be bounded by the size of state space $\lv \S \rv$, the
number of samples $n$ used to learn the model $P$ and {the number of possible transitions}. We
define $\k$ as the largest number of transitions among all states,
\begin{equation}\label{def of kappa}
  \k = \max_{s\in\S} \{k: k = \ll P_s \rl_0, P_s \text{ is the $s$-th row of $P$}\}.
\end{equation}
The following lemma gives the condition for $\rho(\hY) < 1$ with high probability. The proof relies on
the concentration inequality of $l_{1}$-norm of the multinomial distribution.
\begin{lemma}\label{lemma: sample number}
  Under Assumption 1, for any $C>0$ and any positive integer $q>1$, if $n \geq \frac{2C^2\g^2
    \kappa}{\l(1-\g\r)^2}\log{\l( 2\lv \S \rv^{q} \r)}$,
  \begin{equation*}
    \mathbb{P}\left[\p(\hY)  < \frac{1}{C}\right] \geq 1 - \frac{1}{\lv \S \rv^{q-1}}.
  \end{equation*}
\end{lemma}

\begin{proof}
We have
\begin{equation*}
  \p(\hY) \leq \ll \hY \rl_{\infty} \leq \ll \hZ\rl_{\infty}\ll A^{-1} \rl_{\infty} = \frac{\g}{1-\g} \ll P - \hat{P} \rl_{\infty}.
\end{equation*}

By the concentration inequality in \cite{weissman2003inequalities,
  DBLP:journals/corr/abs-2001-11595}, for arbitrary $r \in [0,1]$
\begin{equation}\label{eqn: consistency assumption for row of P}
  \mathbb{P}\left[ \ll e_{i}^{\top} \left(\hat{P} - P  \right) \rl_1 \geq \frac{\sqrt{2\kappa\log{2/r}}}{\sqrt{n}}  \right] \leq r.
\end{equation}
Taking union bound and setting $r = \frac{1}{\lv \S \rv^{q}}$ leads to
\begin{equation*}
  \mathbb{P}\left[\ll P - \hat{P} \rl_{\infty} \geq \frac{\sqrt{2\kappa\log{2\lv \S
          \rv^{q}}}}{\sqrt{n}} \right] \leq 1 - \l(1-1/\lv \S \rv^{q}\r)^{\lv \S \rv} \leq 1/\lv \S
  \rv^{q-1},
\end{equation*}
where the second inequality is by Bernoulli's inequality: for \(r \geq 1\) and \(x \leq 1\),
\[
(1-x)^{r} \geq 1-rx.
\]
The proof is completed
by noticing
\begin{equation*}
  \frac{\sqrt{2\kappa\log{2\lv \S \rv^{q}}}}{\sqrt{n}} \leq \frac{1}{C} \iff n \geq \frac{2C^2\g^2
    \kappa}{\l(1-\g\r)^2}\log{\l( 2\lv \S \rv^{q} \r)}
\end{equation*}
\end{proof}

\begin{remark}\label{rmk: kappa}
  In particular, our goal is to show a bound of $n$ to ensure that $\rho(\hY) < 1$ with high
   probability. In this case, the sample size requirement is
   \begin{equation*}
     n \geq \frac{2\g^2 \kappa}{\l(1-\g\r)^2}\log{\l( 2\lv \S \rv^{q} \r)}.
   \end{equation*}
   The requirement of sample size $n$ only grows at the rate of $O\l(\kappa\log(\lv \S
   \rv)\r)$. Even though $\kappa$ may grow proportionally to $\lv \S \rv$, one can generally assume
   that $\kappa$ grows sublinearly with respect to $\lv \S \rv$. In practice, the numerical examples
   are more well-behaved if $\kappa$ is large, and usually the convergence of the Taylor series
   needs only $n = 1$. The bound on the spectral radius is intended for ill-behaved MDP with small
   $\kappa$.
\end{remark}



          \end{document}